\documentclass{article}

\usepackage{amssymb}
\usepackage{natbib}
\usepackage{color}
\usepackage{bbm}
\usepackage[hidelinks,colorlinks=true,linkcolor=blue,citecolor=blue,breaklinks=true]{hyperref}    
\usepackage{amsfonts}       
\usepackage{xurl}
\usepackage{url}
\usepackage{amsmath}
\usepackage{amsthm}

\usepackage{wrapfig}
\usepackage{mathtools}
\usepackage{tikz}
\usepackage{subfig}
\usepackage{algorithmic}
\usepackage{subfig}
\usepackage{footmisc}
\usepackage{bibentry}
\usepackage{thmtools, thm-restate}
\nobibliography*
\usepackage{mathtools}

\usepackage{physics}
\usepackage{mathrsfs}
\mathtoolsset{showonlyrefs}
\usepackage{etoolbox}
\usepackage{makecell}
\usepackage[ruled]{algorithm2e}
\DontPrintSemicolon
\usepackage{enumerate}
\usepackage{booktabs}
\usepackage{pifont}% http://ctan.org/pkg/pifont
\usepackage{soul}
\usepackage{ulem}
\usepackage{cancel}
\usepackage{microtype}
\newcommand{\fS}{\mathcal{S}}
\newcommand{\fA}{\mathcal{A}}

\newcommand{\fF}{\mathcal{F}}
\newcommand{\fO}{\mathcal{O}}

\newcommand{\R}[1][]{\mathbb{R}^{#1}}

\newcommand{\E}{\mathbb{E}}

\newcommand{\explain}[1]{\tag*{(#1)}}
\newcommand{\cmark}{\ding{51}}%

\newcommand{\inner}[2]{\langle #1, #2 \rangle}

\usepackage{pdfrender}

\allowdisplaybreaks

\newcounter{assucounter}
\numberwithin{assucounter}{section}
\newtheorem{assumption}[assucounter]{Assumption}

\newtheorem{condition}[assucounter]{Condition}

\newtheorem{remark}{Remark}
\newtheorem{theorem}{Theorem}%[section]
\newtheorem{proposition}{Proposition}
\newtheorem{lemma}{Lemma}
\newtheorem{corollary}{Corollary}

\usepackage[main, final]{neurips_2026}

\usepackage[utf8]{inputenc} % allow utf-8 input
\usepackage[T1]{fontenc}    % use 8-bit T1 fonts
\usepackage{url}            % simple URL typesetting
\usepackage{booktabs}       % professional-quality tables
\usepackage{amsfonts}       % blackboard math symbols
\usepackage{nicefrac}       % compact symbols for 1/2, etc.
\usepackage{microtype}      % microtypography
\usepackage{xcolor}         % colors

\title{Almost Sure Convergence Rates of Stochastic Approximation and Reinforcement Learning \\ via a Poisson-Moreau Drift}

\author{%
Xinyu Liu \\
University of Virginia\\
\texttt{xinyuliu@virginia.edu}
\And
Zixuan Xie \\
University of Virginia\\
\texttt{xie.zixuan@email.virginia.edu}
\And
Shangtong Zhang \\
University of Virginia\\
\texttt{shangtong@virginia.edu}
}
\begin{document}

\maketitle

\begin{abstract}
Establishing almost sure convergence rates for stochastic approximation and reinforcement learning under Markovian noise is a fundamental theoretical challenge.
We make progress towards this challenge for a class of stochastic approximation algorithms whose expected updates are contractive, a setting that arises in many reinforcement learning algorithms such as $Q$-learning and linear temporal difference learning.
Specifically,
for a power-law learning rate $\fO(n^{-\eta})$ with $\eta \in (1/2, 1)$,
we obtain an almost sure convergence rate arbitrarily close to $o(n^{1 - 2\eta})$.
For   the  harmonic learning rate   $\alpha_n=\alpha/(n+1)$, the almost sure squared-error rate is $o(n^{-\zeta})$ for every $\zeta<\min\{1,2(1-\kappa)\alpha\}$, where $\kappa$ is the contraction factor. Thus exponents arbitrarily close to one are available when $2(1-\kappa)\alpha\geq1$. This approaches the $n^{-1}$ polynomial exponent of classical law-of-the-iterated-logarithm benchmarks.
 Key to our analysis is a novel Lyapunov drift construction that applies a Poisson-equation based correction for Markovian noise to the well-established Moreau-envelope smoothing for the contractive mapping.
\end{abstract}

\section{Introduction}
\label{sec:intro}
Introduced by \citet{robbins1951stochastic}, stochastic approximation (SA) remains a foundational framework for learning from continuous streams of noisy observations. The canonical recursion,
\begin{align}
    \textstyle\theta_{n+1} = \theta_n + \alpha_n F(\theta_n, Y_n),
\end{align}
where $\theta_n \in \mathbb{R}^d$ is the active estimate and $\alpha_n$ is a decaying learning rate, is remarkably versatile. It natively captures stochastic gradient descent (SGD) algorithms \citep{robbins1951stochastic} as well as core reinforcement learning (RL, \citet{sutton2018reinforcement}) algorithms like $Q$-learning \citep{watkins1992q} and linear temporal difference (TD) learning \citep{sutton1988learning}. Over the decades, a rich classical theory has been developed to answer the basic consistency question: does the sequence $\{\theta_n\}$ converge? While convergence can be analyzed in probability or in $L^p$, our focus is strictly on almost sure convergence. This mode is particularly vital because it guarantees the limiting behavior of almost every individual sample path, offering a stronger, trajectory-wise assurance than average-case or tail-probability metrics.

Historically, almost sure convergence is established via the ODE method, see, e.g., \citet{benveniste1990MP,kushner2003stochastic,borkar2009stochastic,borkar2025ode,liu2025ode}.
While powerful, these techniques typically fall short of quantifying exactly how fast a realized trajectory approaches its limit, i.e., they establish only asymptotic almost sure convergence without a rate. An alternative path leverages the almost-supermartingale framework pioneered by \citet{robbins1971convergence}. While their original work did not immediately extract convergence rates, a recent wave of literature has successfully adapted the almost-supermartingale approach to derive explicit almost sure rates under various conditions \citep{liu2024almost,karandikar2024convergence,qian2024almost,liu2025extensions}. Beyond supermartingale arguments, researchers have derived almost sure rates by imposing stricter structural assumptions. In the realm of linear SA, for example, \citet{chong1999noise} and \citet{tadic2004almost} handle Markovian sequences by exploiting the product-of-matrices dynamics unique to affine updates, whereas \citet{kouritzin2015convergence} achieve similar linear results assuming independent noise. When moving to general nonlinear operators, however, the literature generally bottlenecks at the noise assumptions. \citet{koval2003law} and \citet{vidyasagar2023convergence} establish fundamental convergence rate guarantees, yet their analyses heavily rely on the noise $\{Y_n\}$ being fully independent. \citet{karandikar2024convergence} attempt to relax this to allow non-i.i.d. sequences; however, their framework mandates that the conditional bias decays fast enough to remain summable, an assumption that fundamentally breaks down under standard Markovian noise where the mixing rate is constant. Within the specific domain of RL, notable early milestones include \citet{szepesvari1997asymptotic} for $Q$-learning and \citet{tadic2002almost} for linear TD. Yet, \citet{szepesvari1997asymptotic} relies heavily on count-based learning rates, making it an essentially i.i.d. sampling protocol.

To summarize, although there are many results on almost sure convergence rates for SA and RL,
if we consider two desiderata that are characteristic of modern RL algorithms, namely
\begin{itemize}
  \item[(i)] the noise $\qty{Y_n}$ forms a Markov chain, and
  \item[(ii)] the expectation of $F$ is not necessarily linear and is contractive only in some possibly non-Euclidean norm such as $\ell_\infty$,
\end{itemize}
the only existing result that handles both, to the best of our knowledge, is the recent work \citet{qian2024almost},
which establish almost sure convergence rates for arbitrary-norm contractive stochastic approximations with Markovian noise via a skeleton-iterates technique.
However,
\citet{qian2024almost} cover only a limited range of learning rates and their convergence rates are slower than ours.
Tables~\ref{tab:intro-comparison-all}~\&~\ref{tab:intro-comparison-qian} contextualize the strength of our results with the existing literature about almost sure convergence rates.
In short, our key contribution is \textbf{sharper almost sure convergence rates for general stochastic approximations under both (i) and (ii).}
Our results apply to general contractive Markovian SA under uniform Poisson regularity.

 \begin{table}[h]
\centering
\begin{tabular}{@{}p{0.43\linewidth}p{0.34\linewidth}c@{}}
\toprule
& Algorithm type (i.e., $F$)
& Markovian $\{Y_n\}$\\
\midrule
\citet{pelletier1998almost}
& SGD
& \\
\citet{godichon2016estimating}
& SGD
& \\
\citet{godichon2019lp}
& Averaged SGD
& \\
\citet{sebbouh2021almost}
& SGD / Stochastic Heavy Ball
& \\
\citet{liu2022almost}
& SGD
& \\
\citet{liang2023almost}
& Mirror descent
& \\
\citet{liu2024almost}
& SGD
& \\
\citet{weissmann2025almost}
& SGD
& \\
\midrule
\citet{szepesvari1997asymptotic}
& $Q$-learning
& \\
\citet{tadic2002almost}
& Linear TD
& \cmark \\
\citet{chong1999noise}
& Linear
& \cmark \\
\citet{tadic2004almost}
& Linear
& \cmark \\
\citet{kouritzin2015convergence}
& Linear
& \\
\citet{koval2003law}
& General
& \\
\citet{vidyasagar2023convergence}
& General
& \\
\citet{karandikar2024convergence}
& General
& \\
\citet{qian2024almost}
& General
& \cmark \\
\textbf{This work}
& General
& \cmark \\
\bottomrule
\end{tabular}
\caption{\label{tab:intro-comparison-all} Comparison of works that establish almost sure convergence rates for stochastic approximation algorithms.
By ``General'',
we mean a setting that allows for nonlinear $F$ and arbitrary-norm contractivity.
}
\end{table}
\begin{table}[t]
\centering
\begin{tabular}{ccc}
\toprule
{learning rate regime}
  & \citet{qian2024almost}
  & \textbf{This work} \\
\midrule
$1/2<\eta\le 2/3$
  & N/A
  & a.c. to $o(n^{-(2\eta - 1)})$ \\
$2/3<\eta<1$
  & a.c. to $o(n^{-(\tfrac{3}{2}\eta - 1)})$
  & a.c. to $o(n^{-(2\eta - 1)})$ \\
$\eta=1$
  & $o\qty(\frac{1}{\exp(\zeta_1 \log^{1/(1+\zeta_2)} n)}), \zeta_1 \in (0, 1), 0 < \zeta_2$
  & a.c. to $o(n^{-1})$ \textsuperscript{$\dagger$}  \\
\bottomrule
\end{tabular}
\caption{
\label{tab:intro-comparison-qian}
Comparison with \citet{qian2024almost} for learning rates of $\alpha_n = \fO(n^{-\eta})$.
 The entries describe squared-error rates.
 By ``a.c.'', we mean ``arbitrarily close''.
For example, ``a.c. to $o(n^{-\zeta})$'' means that the rate holds for $o(n^{-\zeta + \epsilon})$ for arbitrary $\epsilon > 0$.
 \textsuperscript{$\dagger$}For the harmonic learning rate $\alpha_n=\alpha/(n+1)$, the displayed ``a.c. to $o(n^{-1})$'' rate requires $2(1-\kappa)\alpha\geq1$; in general, Theorem~\ref{thm:moreau-power-rates} permits every $\zeta<\min\{1,2(1-\kappa)\alpha\}$.
 To understand the rate of \citet{qian2024almost} for $\eta=1$, consider the extreme unattainable case $\zeta_1 = 1$ and $\zeta_2 = 0$,
then the rate degrades to $o(n^{-1})$.
For any $\zeta_1 < 1$ and $\zeta_2 > 0$, the rate is subpolynomial, i.e., it decays slower than $n^{-\zeta}$ for any $\zeta > 0$ because
$\lim_{n\to\infty}{n^\zeta}/{\exp(\zeta_1 \log^{1/(1+\zeta_2)} n)} = \infty$ for any $\zeta > 0$.
  }
\end{table}
From a technical side, we use Moreau-envelope smoothing to address (ii),
which is a well established technique from \citet{chen2020finite,chen2024lyapunov}.
The key technical novelty is that we apply a Poisson-equation based correction to the Moreau envelop based Lyapunov function.
This correction allows us to handle the Markovian noise in (i) by obtaining almost-supermartingales and applying the recent advances in almost-supermartingale convergence rates from \citet{liu2024almost},
generating sharper almost sure convergence rates than \citet{qian2024almost}.
Our construction of the Poisson-Moreau drift is, to the best of our knowledge, novel.

\section{Main Results}
\label{sec:main-results}
We consider the stochastic approximation recursion initialized from a deterministic
$\theta_0\in\R[d]$
\begin{align}
\label{eq:sa}
    \textstyle\theta_{n+1} = \theta_n + \alpha_n F(\theta_n, Y_n). \tag{SA}
\end{align}
Let $(\mathsf Y,\mathcal Y)$ be a measurable state space and let $P$ be a Markov kernel on $\mathsf Y$.
The process $\qty{Y_{n}}_{n\ge0}$ is a Markov chain whose transitions are governed by $P$.
Let $\qty{\mathcal{F}_n}_{n \ge 0}$ be the filtration defined as the $\sigma$-algebra generated by the history of the Markov chain, i.e., $\mathcal{F}_n = \sigma(Y_0, Y_1, \dots, Y_n)$. By construction, $\theta_n$ is $\mathcal{F}_{n-1}$-measurable and $Y_n$ is $\mathcal{F}_n$-measurable. The process satisfies
\begin{align}
\label{eq:bg-markov-property}
    \Pr(Y_{n+1}\in A |\fF_n)=P(Y_n,A)
    \quad\text{a.s. for every } A\in\mathcal Y .
\end{align}
Throughout the paper, $\E_n[\cdot]\doteq \E[\cdot|\fF_n]$.
\begin{assumption}
\label{asp:invariant}
The Markov kernel $P$ admits a unique invariant probability measure $\pi$.
\end{assumption}
Under Assumption~\ref{asp:invariant}, define the expected update $f(\theta)\doteq \int_{\mathsf Y}F(\theta,y)\pi(dy)$.
We consider the setting where the expected update is the residual of some contractive mapping, i.e., we assume $ f(\theta)=T(\theta)-\theta$ for some $T:\R[d]\to\R[d]$ that is contractive under a norm $\norm{\cdot}$.
Or formally,
\begin{assumption}
  Let $\textstyle T(\theta)\doteq \theta+ f(\theta).$
\label{asp:ctct}
Then there exists $\kappa\in(0,1)$ and a norm $\norm{\cdot}$ such that
\begin{align}
\label{eq:T-contraction-star}
    \norm{T(\theta)-T(\theta')}
    \leq \kappa\norm{\theta-\theta'},
    \quad \forall \theta,\theta'\in\R[d].
\end{align}
\end{assumption}
This assumption allows the contraction norm $\norm{\cdot}$ to be non-Euclidean, for example $\ell_\infty$.
We use $\theta_*$ to denote the unique fixed point of $T$, i.e., $ f(\theta_*)=0$.
Such contraction driven stochastic approximations arise in many RL algorithms, such as $Q$-learning and TD.
Notably, it also arises in linear stochastic approximation with Hurwitz matrix, i.e.,
when $ f(\theta) = A\theta + b$ with $A$ being a Hurwitz matrix.
See Lemma 6.1 of \citet{chen2025concentration} for more details on how to verify the contraction condition for such linear stochastic approximation.
This inclusion of linear stochastic approximation with Hurwitz matrix is important because it covers many RL algorithms such as linear TD and gradient TD \citep{sutton2009convergent,sutton2009fast}.
We further assume that $F$ is Lipschitz continuous.
\begin{assumption}
\label{asp:lip}
There exists some constant $L_F$ such that, for $\forall \theta,\theta'\in\R[d]$ and $\forall y\in\mathsf Y$,
\begin{align}
    \norm{F(\theta,y)-F(\theta',y)} \leq L_F\norm{\theta-\theta'},\quad \norm{F(\theta_*,y)}
    \leq L_F. \label{eq:F-lip-s}
\end{align}
\end{assumption}

We then make a mild assumption on the Markovian noise, assuming that it admits a Poisson equation representation.
\begin{assumption}
\label{asp:poisson}
There is a measurable function $H:\R[d]\times\mathsf Y\to\R[d]$ such that for $\forall \theta\in\R[d]$ and $\forall y \in \mathsf Y$,
\begin{align}
\label{eq:poisson}
    H_\theta(y)-(PH_\theta)(y)=F(\theta,y)- f(\theta).
\end{align}
Here $H_\theta(y)$ is a shorthand for $H(\theta,y)$ and
\begin{align}
\label{eq:bg-P-operator}
    \textstyle (PH_\theta)(y)\doteq \int_{\mathsf Y}H_\theta(y') P(y,dy'),
    \quad y\in\mathsf Y .
\end{align}
We further assume there exists some constant $L_H$ such that, for $\forall \theta,\theta'\in\R[d]$ and $\forall y\in\mathsf Y$,
\begin{align}
    \norm{H_\theta(y)-H_{\theta'}(y)}
    \leq L_H\norm{\theta-\theta'}, \quad \norm{H_{\theta_*}(y)}
    \leq L_H. \label{eq:H-lip-s}
\end{align}
\end{assumption}
\begin{remark}
Assumption~\ref{asp:poisson} is standard in the stochastic approximation literature (see, e.g., \citet{benveniste1990MP,kushner2003stochastic,borkar2025ode}).
 The algorithm does not compute $H_\theta$; verifying existence and the stated bounds suffices.
 This assumption is relatively mild. In particular, for any finite state space $\mathcal{Y}$ where the Markov chain $\qty{Y_n}$ is irreducible and aperiodic, the assumption is automatically satisfied for an arbitrary function $F$. Under such finite-state conditions  and Assumption~\ref{asp:lip} , the solution $H_\theta$ can be explicitly constructed via the fundamental matrix (see, e.g., Section 8.2.3 of \citet{puterman2014markov})  and satisfies the required boundedness at $\theta_*$ and Lipschitz continuity in $\theta$ .
For generic state spaces, various structural conditions suffice to establish this property.

For example, uniform geometric ergodicity, $\norm{P^t g}_\infty\leq C_{\rm mix}\mu^t\norm{g}_\infty$ for bounded centered $g$, gives
\begin{align}
\label{eq:revision-poisson-series}
\textstyle H_\theta(y)=\sum_{t=0}^{\infty}P^t\bigl(F(\theta,\cdot)-f(\theta)\bigr)(y),
\qquad L_H=\mathcal{O}\!\left(\frac{C_{\rm mix}L_F}{1-\mu}\right),
\end{align}
where $\mu\in(0,1)$ and $\norm{g}_\infty=\sup_y\norm{g(y)}$. For a fixed contraction norm and factor, the mixing rate $\mu$ does not affect the exponent of $n$ in Theorem~\ref{thm:moreau-power-rates}.

  For instance, the assumption holds if $\qty{Y_n}$ is a positive Harris chain\footnote{The formal definitions of positive and Harris chains are provided on pages 204 and 235 of \citet{meyn2012markov}.} that satisfies the Lyapunov drift condition (V4)\footnote{See page 386 of \citet{meyn2012markov} for the precise definition of the (V4) condition.} with a constant function $V(y) \equiv 1$\footnote{We note that finite irreducible and aperiodic chains trivially satisfy this uniform drift condition.}. Under these conditions, Theorem 2.3 of \citet{glynn1996liapounov} guarantees the assumption through the boundedness of the corresponding fundamental kernel.  On unbounded state spaces, weighted Poisson bounds must imply our uniform bounds.  Elaborating on the complete measure-theoretic prerequisites for these bounds would deviate from the primary focus of this paper. Consequently, we directly state Assumption~\ref{asp:poisson} to maintain focus on our main convergence results.
\end{remark}

 We now state our choice of learning rates and the main convergence result.
\begin{assumption}
\label{asp:lrn_rt}
$\alpha_n=\frac{\alpha}{(n+1)^\eta}$, where $\alpha>0$ and $\frac{1}{2}<\eta\leq1$.
\end{assumption}
\begin{theorem}
\label{thm:moreau-power-rates}
Let Assumptions~\ref{asp:invariant} - \ref{asp:lrn_rt} hold.
The iterates $\qty{\theta_n}$ generated by~\eqref{eq:sa} satisfy
\begin{align}
    \textstyle\lim_{n\to\infty}(n+1)^\zeta\norm{\theta_n-\theta_*}^2 = 0 \qq{a.s.}
\end{align}
Here, if $\frac{1}{2}<\eta<1$, the above holds for any $\zeta \in (0, 2\eta - 1)$; if $\eta=1$, the above holds for any $\zeta \in (0, \min\{1, 2(1-\kappa)\alpha\})$.
\end{theorem}
The proof is given in Section~\ref{proof:moreau-power-rates}.  Here, the restriction $\eta>1/2$ ensures $\sum_n\alpha_n^2<\infty$, as required to control the noise residuals in our almost-supermartingale proof.

\section{Poisson-Moreau Drift and Proof of Theorem~\ref{thm:moreau-power-rates}}
\label{sec:main-proof}

In this section, we illustrate the key technical novelty of this paper.
Write $e_n \doteq \theta_n-\theta_*$.
A common approach to analyze the convergence of $e_n$ is to construct a drift function $V:\R[d]\to\R_+$ and establish the recursion between $V(e_{n+1})$ and $V(e_n)$.
If the norm in Assumption~\ref{asp:ctct} is Euclidean, then the natural choice is $V(x)=\frac{1}{2}\norm{x}^2$, see, e.g., \citet{bhandari2018finite,srikant2019finite}.
To handle a possibly non-Euclidean norm,
Moreau-envelope smoothing is a standard technique.
To our knowledge, the first to use Moreau envelopes in the context of stochastic approximation is \citet{chen2020finite}.
And it has since been widely used in the literature on arbitrary norm contractive stochastic approximation,
see, e.g., \citet{chen2021lyapunov,zhang2022global,chen2024lyapunov,qian2024almost,liu2025linearq,chen2025concentration}.
Below we provide a brief background on this technique.
We do note that these are not our contributions and we provide them here only for completeness.
That being said,
we will shortly articulate why existing developments on Moreau envelopes are not sufficient for our purposes and how we need to further develop the technique by adding a Poisson-equation based correction to the Moreau envelope.

\subsection{Background on Moreau Envelopes}
Denote the Euclidean norm as $\norm{\cdot}_2$.
Thanks to norm equivalence in finite dimensions, we can choose constants $\ell,u>0$ such that
\begin{align}
\label{eq:bg-cs-equivalence}
    \textstyle\ell \norm{x}_2\leq \norm{x}\leq u\norm{x}_2,
    \quad x\in\R[d].
\end{align}
For $\xi>0$, define the Moreau envelope of $x\mapsto \frac{1}{2}\norm{x}^2$ by
\begin{align}
\label{eq:bg-moreau-def}
    \textstyle M_\xi(x) \doteq \inf_{u\in\R[d]}\qty{
        \frac{1}{2}\norm{u}^2+\frac{1}{2\xi}\norm{x-u}_2^2}.
\end{align}
The following are standard facts that we will use,
with $\inner{\cdot}{\cdot}$ denoting the Euclidean inner product.

\begin{lemma}[Proposition A.1 and Section A.2 of \citet{chen2021lyapunov}]
\label{lem:moreau-facts}
For every $\xi>0$, the function $M_\xi$ is continuously differentiable and $1/\xi$-smooth with respect to $\norm{\cdot}_2$:
\begin{align}
\label{eq:moreau-smoothness}
    \textstyle \norm{\nabla M_\xi(x)-\nabla M_\xi(y)}_2
    \leq \frac{1}{\xi}\norm{x-y}_2,
    \quad x,y\in\R[d].
\end{align}
Moreover, there exists a norm $\norm{\cdot}_{m,\xi}$ such that $M_\xi(x)=\frac{1}{2}\norm{x}_{m,\xi}^2$.
With $\ell_\xi\doteq \sqrt{1+\xi\ell^2}, u_\xi\doteq \sqrt{1+\xi u^2}$,
we have
\begin{align}
\label{eq:m-star-equivalence}
    \ell_\xi\norm{x}_{m,\xi}
    \leq \norm{x}
    \leq u_\xi\norm{x}_{m,\xi},
    \quad x\in\R[d].
\end{align}
Finally, for all $x,y\in\R[d]$,
\begin{align}
\label{eq:moreau-gradient-cauchy}
    \abs{\inner{\nabla M_\xi(x)}{y}} \leq \norm{x}_{m,\xi}\norm{y}_{m,\xi},
    \quad
    \inner{\nabla M_\xi(x)}{x}\geq \norm{x}_{m,\xi}^2 .
\end{align}
\end{lemma}

We choose $\xi>0$ small enough that the contraction survives the smoothing step:
\begin{align}
\label{eq:xi-contraction-choice}
    \textstyle\kappa\frac{u_\xi}{\ell_\xi}<1,
\end{align}
and define the resulting smooth drift rate by
\begin{align}
\label{eq:mu-xi-def}
    \textstyle\mu_\xi\doteq 2\left(1-\kappa\frac{u_\xi}{\ell_\xi}\right)>0.
\end{align}
Such a choice is possible because $u_\xi/\ell_\xi\to1$ as $\xi\downarrow0$.

\begin{lemma}
\label{lem:moreau-contract-drift}
Under Assumption~\ref{asp:ctct} and the choice \eqref{eq:xi-contraction-choice}, for every $\theta\in\R[d]$,
\begin{align}
\label{eq:moreau-mean-drift}
  \inner{\nabla M_\xi(\theta-\theta_*)}{ f(\theta)}
  \leq -\mu_\xi M_\xi(\theta-\theta_*).
\end{align}
\end{lemma}
The proof is in Appendix~\ref{proof:moreau-contract-drift}.
This lemma is how the Moreau envelope is typically used: it converts the contraction of $T$ in $\norm{\cdot}$ into a differentiable one-step drift for the smooth energy $M_\xi$.

\subsection{Limitations of Prior Moreau-Based Analyses}
\label{sec:moreau-limitations}
In prior works, the drift function is typically taken to be $V(x) = M_\xi(x)$ directly \citep{chen2020finite, chen2025concentration}.
This is sufficient when $\qty{Y_n}$ are i.i.d.
But when $\qty{Y_n}$ is a Markov chain,
additional terms appear and require additional control.
There are two typical approaches to control these terms resulting from Markovian noise.
The first is based on \citet{srikant2019finite},
which constructs an auxiliary stationary Markov chain to compare with,
using the mixing property of the Markov chain to show that the bias terms are small on average,
see, e.g., \citet{chen2021lyapunov,zhang2022global,chen2024lyapunov}.
Essentially, this approach will establish the recursion between $\E\qty[M_\xi(e_{n+1})]$ and $\E\qty[M_\xi(e_n)]$.
This is sufficient for obtaining convergence rates in expectation.
But it does not give pathwise control of the bias terms and thus is not sufficient for obtaining almost sure rates.
The second approach is the skeleton iterates approach from \citet{qian2024almost}.
Instead of working with the original timescale $\qty{1, 2, \dots}$,
\citet{qian2024almost} carefully choose a subsequence of time points $\qty{n_k}$ and investigate $\qty{\theta_{n_k}}$ instead of $\qty{\theta_n}$.
This allows to construct the recursion between $\E_{n_k}\qty[M_\xi(e_{n_{k+1}})]$ and $M_\xi(e_{n_k})$.
With this recursion,
\citet{qian2024almost} show that $\qty{M_\xi(e_{n_k})}$ is almost a supermartingale and apply the almost supermartingale convergence results from \citet{liu2024almost} to obtain almost sure convergence rates for $\qty{M_\xi(e_{n_k})}$,
which is then propagated back to $\qty{M_\xi(e_n)}$.
The foundation of this approach is recently formally verified by \citet{zhang2026towards}.
However, as summarized in Table~\ref{tab:intro-comparison-qian},
the   skeleton-iterate result requires $\eta>2/3$ and yields squared-error exponents $\zeta<3\eta/2-1$ for $2/3<\eta<1$, whereas our result covers $\eta>1/2$ and gives $\zeta<2\eta-1$. Its harmonic guarantee is subpolynomial .
In this paper, we provide a third approach that applies a Poisson-equation based correction to the Moreau drift directly.

\subsection{Poisson Correction}
The key technical novelty of this work is the following drift
\begin{align}
\label{eq:moreau-V-def}
    V_n^\xi
    \doteq M_\xi(e_n)
    +\alpha_n\inner{\nabla M_\xi(e_n)}{H_{\theta_n}(Y_n)}
    +K\alpha_n^2.
\end{align}
Here $K$ is a positive constant to be chosen large enough.
The middle term is the Poisson correction.
It has the same order as one update step.
So it does not change the asymptotic size of the drift but cancels a leading Markovian term in the conditional drift.
 For i.i.d. sampling, the Poisson correction is unnecessary, and the drift reduces to $M_\xi(e_n)+K\alpha_n^2$.
 The last term is a technical correction to control the residual terms that appear after the Poisson correction.
We then show that $V_n^\xi$ is almost a nonnegative supermartingale.

\begin{lemma}
\label{lem:conditional-chandak-moreau-poisson}
Under Assumptions~\ref{asp:invariant}, \ref{asp:lip}, and~\ref{asp:poisson}, fix $\xi>0$,
there exists some constant $C_\xi<\infty$ such that
\begin{align}
\E_n \qty[M_\xi(e_{n+1})] \leq& M_\xi(e_n)
+\alpha_n\inner{\nabla M_\xi(e_n)}{f(\theta_n)}\\
& +\alpha_n \E_n\qty[
\inner{\nabla M_\xi(e_n)}{H_{\theta_n}(Y_n)}
- \inner{\nabla M_\xi(e_{n+1})}{H_{\theta_{n+1}}(Y_{n+1})}] \\
& + C_\xi\alpha_n^2\bigl(\norm{e_n}_{m,\xi}^2+1\bigr).
\end{align}
\end{lemma}
The proof is in Appendix~\ref{proof:conditional-chandak-moreau-poisson}.
\begin{lemma}
\label{lem:moreau-shifted-drift}
Under Assumptions~\ref{asp:invariant} - \ref{asp:lrn_rt} and the choice \eqref{eq:xi-contraction-choice}, there exists $K_\xi<\infty$ such that, for every $K\geq K_\xi$, there exist constants $C_{\xi,K}<\infty$ and $N_{\xi,K}<\infty$ such that, for every $n\geq N_{\xi,K}$,
\begin{align}
    V_n^\xi
    &\geq \frac{1}{4}\norm{e_n}_{m,\xi}^2\geq0,
    \label{eq:moreau-V-coercive}\\
    \E_n[V_{n+1}^\xi]
    &\leq \bigl(1-\mu_\xi\alpha_n+C_{\xi,K} r_n\bigr)V_n^\xi+C_{\xi,K} r_n,
    \quad\text{a.s.,}
    \label{eq:moreau-V-drift}
\end{align}
where $r_n \doteq \alpha_n^2+\abs{\alpha_{n+1}-\alpha_n}$.
\end{lemma}
The proof is in Appendix~\ref{proof:moreau-shifted-drift}.
As established in Lemma~\ref{lem:moreau-shifted-drift}, the shifted energy $V_{n}^{\xi}$ satisfies an almost-supermartingale drift inequality. Therefore, we can directly invoke almost-supermartingale convergence results to obtain the almost sure convergence rates for $V_{n}^{\xi}$, which then translate to the original error $\norm{\theta_n - \theta_*}$. The detailed proof of Theorem~\ref{thm:moreau-power-rates} is deferred to Appendix \ref{proof:moreau-power-rates}.

\subsection{Limitations of Prior Poisson-Based Analyses}
\label{sec:poisson_limitations}

The Poisson equation is a standard tool to decompose Markovian noise (cf. RHS of~\eqref{eq:poisson}) into a martingale difference (cf. LHS of~\eqref{eq:poisson}) plus a telescoping remainder, enabling ODE-based convergence proofs \citep{benveniste1990MP,kushner2003stochastic,borkar2025ode}.
More recently, the Poisson equation also appears in convergence rates analysis of RL algorithms. \citet{chandak2026concentration} use the Poisson equation to handle Markovian noise in deriving concentration bounds for linear TD.
\citet{blaser2026asymptotic} develop novel Poisson-equation-based bounds on noise terms for nonexpansive SA with Markovian noise, proving almost sure convergence of tabular average-reward TD. \citet{chandak2025finite} combine Moreau envelopes with Poisson equation solutions to obtain finite-time mean-square bounds for two-timescale SA with arbitrary norm contractions and Markovian noise.
What we perceive as common across all these classical and recent works is \emph{telescoping}.
Namely, obtaining Lemma~\ref{lem:conditional-chandak-moreau-poisson} through Poisson equation decomposition is entirely a standard approach and is not novel at all.
Specifically, Lemma~\ref{lem:conditional-chandak-moreau-poisson} can be viewed as a one-timescale specialization of the two-timescale recursion Lemma~3(a) of \citet{chandak2025finite}.
Lemma~\ref{lem:conditional-chandak-moreau-poisson} is interesting because we can further decompose the Poisson difference term $\alpha_n \inner{\nabla M_\xi(e_n)}{H_{\theta_n}(Y_n)} - \alpha_n \inner{\nabla M_\xi(e_{n+1})}{H_{\theta_{n+1}}(Y_{n+1})}$ into a telescoping term $\alpha_n \inner{\nabla M_\xi(e_n)}{H_{\theta_n}(Y_n)} - \alpha_{n+1} \inner{\nabla M_\xi(e_{n+1})}{H_{\theta_{n+1}}(Y_{n+1})}$ plus a residual term $(\alpha_{n+1} - \alpha_n) \inner{\nabla M_\xi(e_{n+1})}{H_{\theta_{n+1}}(Y_{n+1})}$.
The telescoping term will then cancel when we perform summation over $n$.
However, this nice telescoping is not directly available because the terms are inside a conditional expectation indexed by $n$ and thus cannot be directly telescoped.
The typical approach to handle this in all the above works is to take total expectations.
The $n$-indexed conditional expectation is then removed , and  we can telescope $\alpha_n \E\qty[\inner{\nabla M_\xi(e_n)}{H_{\theta_n}(Y_n)}] - \alpha_{n+1} \E\qty[\inner{\nabla M_\xi(e_{n+1})}{H_{\theta_{n+1}}(Y_{n+1})}]$ directly.
This is sufficient for obtaining convergence rates in expectation,
but it does not give pathwise control of the bias terms and thus is not sufficient for obtaining almost sure rates.
Our key novelty is instead to directly incorporate the Poisson correction into the drift function via the novel construction of the Poisson-Moreau drift $V_n^\xi$ in~\eqref{eq:moreau-V-def},
allowing us to obtain an almost-supermartingale structure that preserves the conditional expectation.
What is particularly important is the selection of $K$ such that our drift is nonnegative (cf. \eqref{eq:moreau-V-coercive}), which is a key requirement for applying almost-supermartingale convergence results.

\subsection{A Side Product}

As a byproduct of our analysis,
we can take the unconditional expectation on both sides of~\eqref{eq:moreau-V-drift}.
This generates a convergence rate in $L^2$ for all $\eta \in (0, 1]$.

\begin{corollary}
\label{cor:ms-rates}
Let the assumptions of Lemma~\ref{lem:moreau-shifted-drift} hold but allow any $0 < \eta \le 1$ in the learning rate.
Then the iterates $\qty{\theta_n}$ generated by~\eqref{eq:sa} satisfy
\begin{enumerate}
    \item If $0 < \eta < 1$, then $\E\qty[\norm{\theta_n - \theta_*}^2] = \mathcal{O}(n^{-\eta})$.
    \item If $\eta = 1$, then
    \begin{align}
        \E\qty[\norm{\theta_n - \theta_*}^2] = \begin{cases} \fO(n^{-\min\qty{\mu_\xi \alpha, 1}}), & \text{if } \mu_\xi \alpha \neq 1, \\
        \fO(\frac{\log n}{n}), & \text{if } \mu_\xi \alpha = 1. \end{cases}
    \end{align}
\end{enumerate}

\end{corollary}
The proof is in Appendix~\ref{proof:ms-rates}.
\begin{remark}
We find this corollary interesting because it does not require square-summability of the learning rate (i.e., $0 < \eta \leq 1/2$ is allowed). This is not new because \citet{lauand2024revisiting}
already obtain $L^p$ convergence rates for any $p$ and $\eta \in (0,1)$.
We nevertheless include this corollary because our approach is different from that of
\citet{lauand2024revisiting} and it is conceptually interesting to see that the exact same phenomenon
follows naturally from a short Poisson-Moreau drift argument in the arbitrary-norm contractive
setting.
It is worth noting that the $L^2$ convergence rates in \citet{chandak2025finite} only hold for $\eta \in (1/2, 1]$, according to their theorem statement.
It is not clear to us whether their analysis can eventually cover the full range $\eta \in (0, 1]$.
\end{remark}

\section{Applications}
\label{sec:applications}

Although our theoretical framework applies broadly to contractive RL algorithms, including $Q$-learning and linear TD,
we in this section focus on a  finite-window generalized TD template with bounded linear features and  generalized  importance sampling factors  ; tabular TD is a special case .
This choice is motivated by the fact that almost sure convergence rates for several classical RL algorithms have already been studied using algorithm-specific arguments, whereas pathwise rates for the single-trajectory version of this generalized TD family do not appear to be available.
The recursion we study is the unified TD-learning template considered by \citet{chen2025concentration}.
By choosing the generalized importance sampling factors $c$ and $\rho$, this template recovers on-policy $n$-step TD \citep{sutton2018reinforcement}, off-policy $n$-step TD \citep{precup2000eligibility, sutton2018reinforcement}, $Q^\pi(\lambda)$ \citep{harutyunyan2016q}, Tree-Backup \citep{precup2000eligibility}, Retrace \citep{munos2016safe}, and $Q$-trace \citep{khodadadian2021finite}.
We therefore do not claim novelty in the definition of the recursion or in the structural Hurwitz conditions for this TD family.
Our contribution is instead to combine these structural properties with the Poisson-Moreau drift to obtain pathwise almost sure rates under Markovian single-trajectory sampling.
This is complementary to \citet{chen2025concentration}, who establish maximal high-probability concentration bounds for the same generalized TD template under i.i.d.\ sampling, but do not provide almost sure convergence rates for the Markovian single-trajectory recursion.

Consider an infinite-horizon discounted Markov decision process (MDP, \citet{bellman1957markovian}) with a finite state space $\fS$, a finite action space $\fA$, a discount factor $\gamma\in(0,1)$, a transition function $p: \fS \times \fS \times \fA \to [0, 1]$, and a reward function $r:\fS\times\fA\to\mathbb{R}$.  Let $\pi$ be the target policy and $\pi_b$ be the behavior policy.  We assume strict behavior coverage,
\begin{align}
\label{eq:gtd-coverage}
    \pi_b(a | s)>0,
    \quad \forall (s,a)\in\fS\times\fA,
\end{align}
and the trajectory is generated by $A_t\sim\pi_b(\cdot | S_t)$, $S_{t+1}\sim p(\cdot | S_t,A_t)$.  We assume the behavior state-action chain $(S_t,A_t)_{t\ge0}$ is irreducible and aperiodic, with stationary distribution
\begin{align}
\label{eq:gtd-stationary-sa}
    \textstyle d_b(s,a)\doteq \kappa_b(s)\pi_b(a | s),
    \quad
    d_{\min}\doteq \min_{(s,a)\in\fS\times\fA}d_b(s,a)>0 .
\end{align}
Fix a horizon $N\ge1$ and define the Markovian data window
\begin{align}
\label{eq:gtd-window}
    Y_n\doteq (S_n,A_n,S_{n+1},A_{n+1},\ldots,S_{n+N},A_{n+N}).
\end{align}
The window chain $(Y_n)_{n\ge0}$ is finite-state and has a unique invariant distribution on its recurrent class. Denote it by $\varpi$.
Let $\phi:\fS\times\fA\to\mathbb{R}^d$ be a bounded feature map.  The tabular case is recovered by taking $d=|\fS||\fA|$ and $\phi(s,a)=e_{(s,a)}$ with $e_{(s, a)}$ being the standard basis vector corresponding to the state-action pair $(s,a)$.  Let
$c,\rho:\fS\times\fA\to\mathbb{R}_+$
be bounded generalized importance sampling factors.  For a window $y=(s_0,a_0,\ldots,s_N,a_N)$, write $\phi_i\doteq\phi(s_i,a_i)$, $c_i\doteq c(s_i,a_i)$, and $\rho_i\doteq\rho(s_i,a_i)$.  For $w\in\mathbb{R}^d$, define
\begin{align}
\label{eq:gtd-delta}
    \delta_i(w;y)
    \doteq
    r(s_i,a_i)+\gamma\rho_{i+1}\phi_{i+1}^{\top}w-\phi_i^\top w,
    \quad i=0,\ldots,N-1,
\end{align}
with the convention that an empty product is one, and define the generalized TD increment
\begin{align}
\label{eq:gtd-increment}
    \textstyle g_{c,\rho}(w,y) \doteq \phi_0 \sum_{i=0}^{N-1} \gamma^i\qty(\prod_{j=1}^{i}c_j)\delta_i(w;y).
\end{align}
The single-trajectory generalized TD recursion is
\begin{align}
\label{eq:single-trajectory-gtd}
    \textstyle w_{n+1}=w_n+a_n g_{c,\rho}(w_n,Y_n),
    \quad
    a_n=\frac{a}{(n+1)^\eta}.
\end{align}
Equivalently, it is a linear stochastic approximation.  Define
\begin{align}
\label{eq:gtd-A}
    A_{c,\rho}(y) &\textstyle\doteq \phi_0 \sum_{i=0}^{N-1} \gamma^i\qty(\prod_{j=1}^{i}c_j)
    \qty(\gamma\rho_{i+1}\phi_{i+1}-\phi_i)^\top,\\
\label{eq:gtd-b}
    b_{c,\rho}(y) &\textstyle\doteq -\phi_0
    \sum_{i=0}^{N-1} \gamma^i\qty(\prod_{j=1}^{i}c_j) r(s_i,a_i).
\end{align}
Then
\begin{align}
\label{eq:gtd-linear-sa}
    g_{c,\rho}(w,y)=A_{c,\rho}(y)w-b_{c,\rho}(y).
\end{align}
Let $\bar A_{c,\rho}\doteq \E_{\varpi}[A_{c,\rho}(Y)]$, $\bar b_{c,\rho}\doteq \E_{\varpi}[b_{c,\rho}(Y)]$.
The corresponding mean root is $w^*_{c,\rho}\doteq \bar A_{c,\rho}^{-1}\bar b_{c,\rho}$,
whenever $\bar A_{c,\rho}$ is nonsingular.

The only algorithm-specific structural input is the following standard Hurwitz condition for generalized TD.
\begin{condition}[Generalized TD Hurwitz condition]
\label{cond:gtd-hurwitz}
The matrix $\bar A_{c,\rho}$ is Hurwitz.
\end{condition}
Under Condition~\ref{cond:gtd-hurwitz}, there exist $\beta_{c,\rho}>0$, a norm $\norm{\cdot}_{\rm GTD}$, and a constant $\kappa_{\rm GTD}\in(0,1)$ such that
\begin{align}
\label{eq:gtd-contract}
    T_{\rm GTD}(w)
    \doteq
    w+\beta_{c,\rho}\bigl(\bar A_{c,\rho}w-\bar b_{c,\rho}\bigr)
\end{align}
is a contraction:
\begin{align}
\label{eq:gtd-contract-ineq}
    \norm{T_{\rm GTD}(w)-T_{\rm GTD}(w')}_{\rm GTD}
    \leq
    \kappa_{\rm GTD}\norm{w-w'}_{\rm GTD},
    \quad
    w,w'\in\mathbb{R}^d .
\end{align}
This follows from the standard Lyapunov reduction for Hurwitz linear systems.

 In particular,   let $P_{c,\rho}\succ0$ be the unique solution of
\begin{align}
\label{eq:revision-gtd-lyapunov}
\bar A_{c,\rho}^{\top}P_{c,\rho}+P_{c,\rho}\bar A_{c,\rho}=-I,
\qquad \norm{x}_{\rm GTD}\doteq\sqrt{x^{\top}P_{c,\rho}x}.
\end{align}
This is equivalent to the Euclidean norm, with constants $\sqrt{\lambda_{\min}(P_{c,\rho})}$ and $\sqrt{\lambda_{\max}(P_{c,\rho})}$. Choose  $\beta_{c,\rho}>0$   sufficiently small and set
\begin{align}
\label{eq:revision-gtd-kappa}
\kappa_{\rm GTD}\doteq\norm{P_{c,\rho}^{1/2}(I+\beta_{c,\rho}\bar A_{c,\rho})P_{c,\rho}^{-1/2}}_2<1.
\end{align}
Here $\norm{\cdot}_2$ is the Euclidean operator norm, and $P_{c,\rho}$, $\beta_{c,\rho}$, and $\kappa_{\rm GTD}$ are used only in the analysis .

Condition~\ref{cond:gtd-hurwitz} is already verified for important members of the generalized TD family.  In the on-policy case $\pi_b=\pi$ with $c=\rho=1$, \citet{chen2025concentration} show that $\bar A_{c,\rho}$ is Hurwitz.  In the off-policy case, their generalized TD analysis gives sufficient conditions on $c,\rho$ and the horizon $N$ that imply the same Hurwitz property.

\begin{table}[t]
\centering
\caption{Examples of generalized importance sampling factors covered by the recursion \eqref{eq:single-trajectory-gtd}.  Here $\varrho(s,a)\doteq \pi(a | s)/\pi_b(a | s)$.  The theorem below applies whenever the associated Hurwitz condition holds.}
\label{tab:gtd-special-cases}
\begin{tabular}{lll}
\toprule
Algorithm & $c(s,a)$ & $\rho(s,a)$ \\
\midrule
On-policy $n$-step TD & $1$ & $1$ \\
Off-policy $n$-step TD & $\varrho(s,a)$ & $\varrho(s,a)$ \\
$Q$-trace & $\min\{\bar c,\varrho(s,a)\}$ & $\min\{\bar\rho,\varrho(s,a)\}$ \\
Retrace$(\lambda)$ & $\lambda\min\{1,\varrho(s,a)\}$ & $\varrho(s,a)$ \\
Tree-Backup$(\lambda)$ & $\lambda\pi(a | s)$ & $\varrho(s,a)$ \\
$Q^\pi(\lambda)$ & $\lambda$ & $\varrho(s,a)$ \\
\bottomrule
\end{tabular}
\end{table}

To apply Theorem~\ref{thm:moreau-power-rates} to the actual recursion \eqref{eq:single-trajectory-gtd}, set
\begin{align}
\label{eq:gtd-SA-identification}
    \textstyle \theta_n=w_n,
    \quad
    \theta_*=w^*_{c,\rho},
    \quad
    F_{\rm GTD}(w,y)\doteq
    \beta_{c,\rho}\bigl(A_{c,\rho}(y)w-b_{c,\rho}(y)\bigr),
    \quad
    \alpha_n\doteq \frac{a_n}{\beta_{c,\rho}} .
\end{align}
Then \eqref{eq:single-trajectory-gtd} is exactly in the form \eqref{eq:sa}, and its averaged map is $T_{\rm GTD}$.

\begin{proposition}
\label{prop:gtd-assumptions}
Let \eqref{eq:gtd-coverage} and Condition~\ref{cond:gtd-hurwitz} hold.
Let the behavior state-action chain be irreducible and aperiodic.
Then the single-trajectory generalized TD recursion \eqref{eq:single-trajectory-gtd}, written in the form \eqref{eq:gtd-SA-identification}, satisfies Assumptions~\ref{asp:invariant}, \ref{asp:ctct}, \ref{asp:lip}, and \ref{asp:poisson} with contraction norm $\norm{\cdot}_{\rm GTD}$ and contraction factor $\kappa_{\rm GTD}$.
\end{proposition}
The proof is in Appendix~\ref{proof:gtd-assumptions}.

Combining Proposition~\ref{prop:gtd-assumptions} with Theorem~\ref{thm:moreau-power-rates} gives the following pathwise rate.
\begin{theorem}[Almost sure rate for single-trajectory generalized TD]
\label{thm:gtd-as-rate}
Under the assumptions of Proposition~\ref{prop:gtd-assumptions}, let $a_n=a/(n+1)^\eta$ with $a>0$ and $1/2<\eta\leq1$.  Then
\begin{align}
\label{eq:gtd-rate}
    (n+1)^\zeta
    \norm{w_n-w^*_{c,\rho}}_{\rm GTD}^2
    \longrightarrow 0
    \quad\text{a.s.}
\end{align}
If $1/2<\eta<1$, then \eqref{eq:gtd-rate} holds for every $\zeta\in(0,2\eta-1)$.  If $\eta=1$, then \eqref{eq:gtd-rate} holds for every
\begin{align}
    \textstyle\zeta\in
    \left(0,\min\qty{1,2(1-\kappa_{\rm GTD})\frac{a}{\beta_{c,\rho}}}\right).
\end{align}
\end{theorem}

 Here, for harmonic learning rates, the conclusion of Theorem~\ref{thm:gtd-as-rate} holds for every $\zeta<1$ whenever $2(1-\kappa_{\rm GTD})a/\beta_{c,\rho}\geq1$. No projection or bounded-iterate assumption is required.
Appendix~\ref{app:revision-experiment} provides a numerical illustration of our theoretical results.

 \section{Related Work}

\textbf{Convergence Modes}.
Almost sure convergence rates are in general challenging to obtain.
This can be seen from the fact that a rich literature on RL has been devoted to establishing almost sure convergence of various algorithms \citep{watkins1989learning,watkins1992q,jaakkola1993convergence,tsitsiklis1994asynchronous,tsitsiklis1997analysis,tsitsiklis1999average,konda2000actor,abounadi2001learning,sutton2009convergent,sutton2009fast,maei2011gradient,yu2015convergence,yu2017convergence,zhang2019provably,zhang2020average,wan2020learning,qian2025revisiting,wang2024almost,blaser2026almost}, but only a few works have established almost sure convergence rates (Table~\ref{tab:intro-comparison-all}).
  For classical discounted linear TD, \citet{tadic2002almost} establishes the almost sure squared-error rate $\mathcal{O}((\log\log n)/n)$ under specialized assumptions and harmonic learning rates. This is typically the optimal almost sure squared-error benchmark in classical linear SA and TD settings. Our theorem covers nonlinear mean fields and arbitrary contraction norms under uniform Poisson regularity; if $2(1-\kappa)\alpha\geq1$, harmonic learning rates yield $o(n^{-\zeta})$ almost surely for every $\zeta<1$, approaching the same $n^{-1}$ polynomial exponent.

SA and  RL have also been extensively studied through the lenses of other convergence modes.
Notable works on $L^2$ convergence include \citet{polyak1992acceleration, moulines2011non, bhandari2018finite, srikant2019finite, zou2019finite, mou2020linear, chen2021lyapunov},  \citet{liu2021gradient},  \citet{zhang2021truncated, liu2025linearq, xie2025finite} , \citet{mitra2024simple,lee2025unprojected} .
Notable works on $L^p$ convergence include \citet{godichon2016estimating, godichon2019lp, gupta2019finite, srikant2019finite, qu2020finite,  durmus2021stability, jentzen2021strong, wang2021convergence, chen2022finite, liu2025extensions}.
Notable works on high probability concentration include \citet{korda2015td, dalal2018finite, shah2018q, thoppe2019concentration, wainwright2019stochastic, borkar2021concentration, li2020sample, mou2020linear, qu2020finite, patil2023finite, li2024q, li2024high, samsonov2024improved, durmus2025finite, chandak2025finite, chen2025concentration, khodadadian2025general, chandak2026concentration}.
Notable works on distributional convergence, particularly, (functional) central limit theorem (CLT), include \citet{ruppert1988efficient, benveniste1990MP, polyak1992acceleration, borkar2009stochastic, li2023statistical, hu2024central, samsonov2024gaussian, zhang2024constant, borkar2025ode, srikant2025rates, liu2026central, wu2026statistical}.
This work on almost sure convergence rates provides a precise pathwise characterization and is complementary to these works on other convergence modes.

\textbf{Poisson Equation.}
The foundational theory (e.g., existence of solutions, growth bounds, square integrability) is developed by \citet{glynn1989importance} and \citet{glynn1996liapounov},
who also use these foundations to derive functional CLTs and perturbation results for Markov processes.
\cite{makowski1994poisson} provided a probabilistic treatment of existence, uniqueness, growth estimates, and parametric dependence for the Poisson equation on countable state spaces.
In addition to the use of Poisson equations in SA and RL surveyed in Section~\ref{sec:poisson_limitations},
in the MCMC literature, Poisson equations are also used to construct control variates for variance reduction \citep{dellaportas2012control} and unbiased estimators via coupling \citep{douc2026solving}.

\section{Conclusion}

Through the introduction of a Poisson-corrected Moreau drift,
this work establishes sharper almost sure convergence rates for contractive stochastic approximation under Markovian noise.
A few questions remain open.
First, Corollary~\ref{cor:ms-rates} gives $L^2$ rates for $\eta \in (0, 1]$, but Theorem~\ref{thm:moreau-power-rates} only gives almost sure rates for $\eta \in (1/2, 1]$.
It is not clear to us whether our methodology can yield almost sure rates for the full range $\eta \in (0, 1]$.
Second,
although our almost sure rates are sharper than \citet{qian2024almost},
they are still not tight.
It is not clear to us whether we can fully match the $O(n^{-1}\log\log n)$ rate given by the law of the iterated logarithm.
Third,
in addition to almost sure rates,
\citet{qian2024almost} also establish high probability concentration bounds with exponentially decaying tails for contractive SA with Markovian noise.
It is not clear to us whether our Poisson-Moreau framework can be refined to yield complementary concentration bounds with exponential tails.

 \begin{ack}
This work is supported in part by the US National Science Foundation under the awards III-2128019, SLES-2331904, and CAREER-2442098, the Commonwealth Cyber Initiative's Central Virginia Node under the award VV-1Q26-001, VV-1Q27-007, and VV-1Q27005, 4VA at UVA, a Cisco Faculty Research Award, and a Google Research Award.
\end{ack}

\clearpage

{\small
\bibliography{bibliography}
}

\newpage
\appendix
\section{Auxiliary Lemmas}
\begin{lemma}
  [Discrete Gronwall Inequality, Lemma 8 in Section 11.2 of \citet{borkar2009stochastic}]
\label{lem:discrete_gronwall}
    For non-negative real sequences $\qty{x_n, n \geq 0}$ and $\qty{a_n, n \geq 0}$ and scalar  $L \geq 0$,  it holds
    \begin{align}
        \textstyle x_{n+1} \leq C + L\sum_{i=0}^n a_ix_i \quad \forall n
    \implies
        \textstyle x_{n+1} \leq (C+x_0)\exp({L\sum_{i=0}^n a_i}) \quad \forall n.
    \end{align}
\end{lemma}

\begin{lemma}
\label{lem:weighted-rs}
Let $\qty{X_n}_{n\ge N}$ be nonnegative, adapted, and integrable.  Suppose that
\begin{align}
\label{eq:weighted-rs-rec}
  \E[X_{n+1} |\fF_n]\le (1-\beta_n)X_n+b_n
  \quad\text{a.s. for every }n\ge N,
\end{align}
where $0\le \beta_n\le1$ and $b_n\ge0$ are deterministic.  Let $(q_n)_{n\ge N}$ be a positive deterministic sequence and define
\begin{align}
\label{eq:gamma-def}
  \gamma_n \doteq 1-\frac{q_{n+1}}{q_n}(1-\beta_n).
\end{align}
Assume that, after increasing $N$ if necessary,
\begin{align}
\label{eq:weighted-rs-cond}
  0 \leq \gamma_n\leq 1,
  \quad
  \sum_{n=N}^\infty \gamma_n=\infty,
  \quad
  \sum_{n=N}^\infty q_{n+1}b_n<\infty.
\end{align}
Then
\begin{align}
    q_nX_n\longrightarrow0
  \quad\text{a.s.}
\end{align}
\end{lemma}

\begin{proof}
Multiplying \eqref{eq:weighted-rs-rec} by $q_{n+1}$ and using \eqref{eq:gamma-def}, we obtain
\begin{align}
  \E[q_{n+1}X_{n+1} |\fF_n]
  &\leq q_{n+1}(1-\beta_n)X_n+q_{n+1}b_n\\
  &=(1-\gamma_n)q_nX_n+q_{n+1}b_n\\
  &=q_nX_n-\gamma_n q_nX_n+q_{n+1}b_n .
\end{align}
Since $\sum_n q_{n+1}b_n<\infty$, the Robbins--Siegmund almost-supermartingale theorem \citep{robbins1971convergence} implies that $q_nX_n$ converges almost surely to a finite random variable and
\begin{align}
  \sum_{n=N}^\infty \gamma_nq_nX_n<\infty
  \quad\text{a.s.}
\end{align}
If the almost sure limit of $q_nX_n$ were positive on an event of positive probability, then on that event $q_nX_n$ would be bounded below by a positive constant for all sufficiently large $n$, contradicting $\sum_n\gamma_n=\infty$.  Hence the limit is zero almost surely.
\end{proof}

\begin{lemma}
\label{lem:power-law-drift-rates}
Let
\begin{align}
    \alpha_n=\frac{\alpha}{(n+1)^\eta},
    \quad \alpha>0,
    \quad \frac{1}{2}<\eta\leq1.
\end{align}
Let $\qty{X_n}$ be nonnegative, adapted, and integrable.  Suppose that, on a deterministic tail,
\begin{align}
\label{eq:power-law-drift-assumption}
  \E[X_{n+1} |\fF_n]
  \leq \bigl(1-c_0\alpha_n+C_0r_n\bigr)X_n+C_0r_n
\end{align}
for constants $c_0>0$ and $C_0<\infty$.  Then the following statements hold.
\begin{enumerate}
  \item If $\frac{1}{2}<\eta<1$, then $(n+1)^\zeta X_n\to0$ almost surely for every $0\leq\zeta<2\eta-1$.
  \item If $\eta=1$, then $(n+1)^\zeta X_n\to0$ almost surely for every $0\leq\zeta<\min\{1,c_0\alpha\}$.
\end{enumerate}
\end{lemma}

\begin{proof}
First,
\begin{align}
    \abs{\alpha_{n+1}-\alpha_n}
    =\alpha\bigl((n+1)^{-\eta}-(n+2)^{-\eta}\bigr)
    \leq \alpha\eta(n+1)^{-\eta-1}
    \leq \frac{\eta}{\alpha}\alpha_n^2,
\end{align}
where the last inequality uses $\eta\leq1$.  Thus $r_n\leq C\alpha_n^2$.

Fix any $c\in(0,c_0)$.  Since $\alpha_n\to0$, \eqref{eq:power-law-drift-assumption} implies, after increasing the deterministic tail index,
\begin{align}
\label{eq:appendix-clean-drift}
  \E[X_{n+1} |\fF_n]
  \leq (1-c\alpha_n)X_n+C\alpha_n^2,
  \quad
  0\leq c\alpha_n\leq1 .
\end{align}

Assume first that $\frac{1}{2}<\eta<1$, and let $0\leq\zeta<2\eta-1$.  Apply Lemma~\ref{lem:weighted-rs} to \eqref{eq:appendix-clean-drift} with $q_n=(n+1)^\zeta$, $\beta_n=c\alpha_n$, and $b_n=C\alpha_n^2$.  The weighted error term is summable because
\begin{align}
  \sum_n q_{n+1}b_n
  \leq C\sum_n (n+2)^\zeta(n+1)^{-2\eta}<\infty.
\end{align}
Moreover,
\begin{align}
  \gamma_n
  =1-\frac{q_{n+1}}{q_n}(1-\beta_n)
  =\frac{c\alpha}{(n+1)^\eta}+O((n+1)^{-1}).
\end{align}
Since $\eta<1$, the positive $(n+1)^{-\eta}$ term dominates the $(n+1)^{-1}$ term.  Thus, on a further deterministic tail, $0\leq\gamma_n\leq1$ and $\sum_n\gamma_n=\infty$.  Lemma~\ref{lem:weighted-rs} gives $(n+1)^\zeta X_n\to0$ almost surely.

Now assume $\eta=1$, and let $0\leq\zeta<\min\{1,c_0\alpha\}$.  Choose $c\in(0,c_0)$ so that $\zeta<c\alpha$.  With the same choices of $q_n$, $\beta_n$, and $b_n$, the weighted error term is summable because $\zeta<1$.  Also,
\begin{align}
  \gamma_n
  =1-\left(1+\frac{1}{n+1}\right)^\zeta
      \left(1-\frac{c\alpha}{n+1}\right)
  =\frac{c\alpha-\zeta}{n+1}+O\left(\frac{1}{(n+1)^2}\right).
\end{align}
Since $c\alpha-\zeta>0$, the coefficients $\gamma_n$ are eventually in $[0,1]$ and satisfy $\sum_n\gamma_n=\infty$.  Lemma~\ref{lem:weighted-rs} gives $(n+1)^\zeta X_n\to0$ almost surely.
\end{proof}

\begin{lemma}
\label{lem:shifted-energy-coercive}
Under Assumptions~\ref{asp:lip} and \ref{asp:poisson}, fix $\xi>0$.
There exists $K_\xi<\infty$ such that, for every $K\geq K_\xi$, there exist some constants $C_{\xi,K}'$ and $N_{\xi,K}$ such that, for every $n\geq N_{\xi,K}$, we have
\begin{align}
\label{eq:appendix-coercive}
  V_n^\xi
  &\geq \frac{1}{4}\norm{e_n}_{m,\xi}^2, \\
\label{eq:appendix-V-M-comparison}
  \alpha_n\abs{V_n^\xi-M_\xi(e_n)}
  &\leq C_{\xi,K}'r_n\bigl(\norm{e_n}_{m,\xi}^2+1\bigr).
\end{align}
\end{lemma}

\begin{proof}
By the norm comparison from Lemma~\ref{lem:moreau-facts}, we have $\norm{\cdot}_{m,\xi}\leq \ell_\xi^{-1}\norm{\cdot}$ and $\norm{\cdot} \leq u_\xi\norm{\cdot}_{m,\xi}$. We then obtain
\begin{align}
  \norm{H_{\theta_n}(Y_n)}_{m,\xi}
  \leq& \ell_\xi^{-1}\norm{H_{\theta_n}(Y_n)}\\
  \leq& \ell_\xi^{-1}
  \Bigl(
    \norm{H_{\theta_n}(Y_n)-H_{\theta_*}(Y_n)}
    +\norm{H_{\theta_*}(Y_n)}
  \Bigr)\\
  \leq& \ell_\xi^{-1}L_H\bigl(\norm{e_n}+1\bigr)
  \explain{By Assumption~\ref{asp:poisson}}\\
  \leq& \ell_\xi^{-1}L_H\bigl(u_\xi\norm{e_n}_{m,\xi}+1\bigr)\\
  \leq& L_H\ell_\xi^{-1}(u_\xi+1)
  \bigl(\norm{e_n}_{m,\xi}+1\bigr).
  \label{eq:lemma5-H-growth}
\end{align}
Therefore,
\begin{align}
  \abs{\inner{\nabla M_\xi(e_n)}{H_{\theta_n}(Y_n)}}
  \leq& \norm{e_n}_{m,\xi}\norm{H_{\theta_n}(Y_n)}_{m,\xi}
  \explain{By \eqref{eq:moreau-gradient-cauchy}}\\
  \leq& L_H\ell_\xi^{-1}(u_\xi+1)
  \norm{e_n}_{m,\xi}
  \bigl(\norm{e_n}_{m,\xi}+1\bigr).
  \label{eq:lemma5-poisson-correction-growth}
\end{align}
Since $\alpha_n$ is decreasing, choose a deterministic index $N_\xi^{(1)}$ such that, for every $n\geq N_\xi^{(1)}$,
\begin{align}
\label{eq:lemma5-tail-choice}
  \alpha_n
  &\leq 1,
  \quad
  L_H\ell_\xi^{-1}(u_\xi+1)\alpha_n
  \leq \frac{1}{8}.
\end{align}
Then, for every $n\geq N_\xi^{(1)}$,
\begin{align}
  \alpha_n
  \abs{\inner{\nabla M_\xi(e_n)}{H_{\theta_n}(Y_n)}}
  \leq& L_H\ell_\xi^{-1}(u_\xi+1)\alpha_n
  \norm{e_n}_{m,\xi}^2 + L_H\ell_\xi^{-1}(u_\xi+1)\alpha_n
  \norm{e_n}_{m,\xi} \\
  \leq& \frac{1}{8}\norm{e_n}_{m,\xi}^2
  +L_H\ell_\xi^{-1}(u_\xi+1)\alpha_n
  \norm{e_n}_{m,\xi} \\
  \leq& \frac{1}{8}\norm{e_n}_{m,\xi}^2
  +\frac{1}{8}\norm{e_n}_{m,\xi}^2
  +2L_H^2\ell_\xi^{-2}(u_\xi+1)^2\alpha_n^2
  \explain{Cauchy--Schwarz inequality}\\
  =& \frac{1}{4}\norm{e_n}_{m,\xi}^2
  +2L_H^2\ell_\xi^{-2}(u_\xi+1)^2\alpha_n^2.
  \label{eq:lemma5-alpha-poisson-correction-bound}
\end{align}
Set
\begin{align}
\label{eq:lemma5-K-xi-def}
  K_\xi
  \doteq 2L_H^2\ell_\xi^{-2}(u_\xi+1)^2.
\end{align}
Fix $K\geq K_\xi$ and set $N_{\xi,K}\doteq N_\xi^{(1)}$. For every $n\geq N_{\xi,K}$,
\begin{align}
  V_n^\xi
  =& M_\xi(e_n)
  +\alpha_n\inner{\nabla M_\xi(e_n)}{H_{\theta_n}(Y_n)}
  +K\alpha_n^2\\
  \geq& M_\xi(e_n)
  -\alpha_n\abs{\inner{\nabla M_\xi(e_n)}{H_{\theta_n}(Y_n)}}
  +K\alpha_n^2\\
  =& \frac{1}{2}\norm{e_n}_{m,\xi}^2
  -\alpha_n\abs{\inner{\nabla M_\xi(e_n)}{H_{\theta_n}(Y_n)}}
  +K\alpha_n^2 \\
  \geq& \frac{1}{2}\norm{e_n}_{m,\xi}^2
  -\frac{1}{4}\norm{e_n}_{m,\xi}^2
  -K_\xi\alpha_n^2 +K\alpha_n^2\\
  \geq& \frac{1}{4}\norm{e_n}_{m,\xi}^2.
\end{align}
This proves \eqref{eq:appendix-coercive}.

It remains to prove \eqref{eq:appendix-V-M-comparison}. From \eqref{eq:lemma5-poisson-correction-growth}, we have
\begin{align}
  \abs{\inner{\nabla M_\xi(e_n)}{H_{\theta_n}(Y_n)}}
  \leq& L_H\ell_\xi^{-1}(u_\xi+1)
  \bigl(\norm{e_n}_{m,\xi}^2+\norm{e_n}_{m,\xi}\bigr)\\
  \leq& 2L_H\ell_\xi^{-1}(u_\xi+1)
  \bigl(\norm{e_n}_{m,\xi}^2+1\bigr).
  \label{eq:lemma5-poisson-correction-square-growth}
\end{align}
For every $n\geq N_{\xi,K}$,
\begin{align}
  \alpha_n\abs{V_n^\xi-M_\xi(e_n)}
  =& \alpha_n
  \abs{
    \alpha_n\inner{\nabla M_\xi(e_n)}{H_{\theta_n}(Y_n)}
    +K\alpha_n^2
  }\\
  \leq& \alpha_n^2
  \abs{\inner{\nabla M_\xi(e_n)}{H_{\theta_n}(Y_n)}}
  +K\alpha_n^3\\
  \leq& 2L_H\ell_\xi^{-1}(u_\xi+1)\alpha_n^2
  \bigl(\norm{e_n}_{m,\xi}^2+1\bigr)
  +K\alpha_n^3\\
  \leq& \bigl(2L_H\ell_\xi^{-1}(u_\xi+1)+K\bigr)\alpha_n^2
  \bigl(\norm{e_n}_{m,\xi}^2+1\bigr)
\end{align}
Since $r_n\geq \alpha_n^2$, denote $C_{\xi,K}'
  \doteq 2L_H\ell_\xi^{-1}(u_\xi+1)+K$ then completes the proof.
\end{proof}

\section{Proofs of Section~\ref{sec:main-proof}}
\label{sec:shifted-drift-appendix}

\subsection{Proof of Lemma~\ref{lem:moreau-contract-drift}}
\label{proof:moreau-contract-drift}
\begin{proof}
Since $f(\theta)=T(\theta)-\theta$ and $T(\theta_*)=\theta_*$,
\begin{align}
  \inner{\nabla M_\xi(\theta-\theta_*)}{f(\theta)}
  =\inner{\nabla M_\xi(\theta-\theta_*)}{T(\theta)-\theta_*}
   -\inner{\nabla M_\xi(\theta-\theta_*)}{\theta-\theta_*}.
\end{align}
The Moreau gradient inequality \eqref{eq:moreau-gradient-cauchy}, the norm comparison \eqref{eq:m-star-equivalence}, and the contraction property \eqref{eq:T-contraction-star} give
\begin{align}
  \inner{\nabla M_\xi(\theta-\theta_*)}{T(\theta)-\theta_*}
  &\leq \norm{\theta-\theta_*}_{m,\xi}\norm{T(\theta)-\theta_*}_{m,\xi} \\
  &\leq \ell_\xi^{-1}\norm{\theta-\theta_*}_{m,\xi}\norm{T(\theta)-\theta_*} \\
  &\leq \kappa\frac{u_\xi}{\ell_\xi}\norm{\theta-\theta_*}_{m,\xi}^2 .
\end{align}
Again by \eqref{eq:moreau-gradient-cauchy},
\begin{align}
  \inner{\nabla M_\xi(\theta-\theta_*)}{\theta-\theta_*}
  \geq \norm{\theta-\theta_*}_{m,\xi}^2 .
\end{align}
Therefore,
\begin{align}
  \inner{\nabla M_\xi(\theta-\theta_*)}{f(\theta)}
  &\leq -\left(1-\kappa\frac{u_\xi}{\ell_\xi}\right)
      \norm{\theta-\theta_*}_{m,\xi}^2 \\
  &=-\mu_\xi M_\xi(\theta-\theta_*).
\end{align}
\end{proof}

\subsection{Proof of Lemma~\ref{lem:conditional-chandak-moreau-poisson}}
\label{proof:conditional-chandak-moreau-poisson}

\begin{proof}
Let $\bar\alpha=\sup_{n\geq0}\alpha_n<\infty$. We first collect the bounds used in the proof.
\begin{align}
  \norm{F(\theta_n,Y_n)}
  \leq& \norm{F(\theta_n,Y_n)-F(\theta_*,Y_n)}
  +\norm{F(\theta_*,Y_n)}\\
  \leq& L_F\bigl(\norm{e_n}+1\bigr)
  \explain{By Assumption~\ref{asp:lip}}\\
  \leq& L_F(u_\xi+1)\bigl(\norm{e_n}_{m,\xi}+1\bigr),
  \label{eq:lemma3-F-growth-original}\\
  \norm{F(\theta_n,Y_n)}_2
  \leq& \ell^{-1}\norm{F(\theta_n,Y_n)}
  \leq L_F\ell^{-1}(u_\xi+1)
  \bigl(\norm{e_n}_{m,\xi}+1\bigr),
  \label{eq:lemma3-F-growth-euclidean}\\
  \norm{F(\theta_n,Y_n)}_{m,\xi}
  \leq& \ell_\xi^{-1}\norm{F(\theta_n,Y_n)}
  \leq L_F\ell_\xi^{-1}(u_\xi+1)
  \bigl(\norm{e_n}_{m,\xi}+1\bigr).
  \label{eq:lemma3-F-growth-moreau}
\end{align}
Similarly, by Assumption~\ref{asp:poisson}, for every $y\in\mathsf Y$,
\begin{align}
  \norm{H_{\theta_n}(y)}
  \leq& \norm{H_{\theta_n}(y)-H_{\theta_*}(y)}
  +\norm{H_{\theta_*}(y)}\\
  \leq& L_H\bigl(\norm{e_n}+1\bigr)
  \explain{By Assumption~\ref{asp:poisson}}\\
  \leq& L_H(u_\xi+1)\bigl(\norm{e_n}_{m,\xi}+1\bigr),
  \label{eq:lemma3-H-growth-original}\\
  \norm{H_{\theta_n}(y)}_2
  \leq& \ell^{-1}\norm{H_{\theta_n}(y)}\\
  \leq& L_H\ell^{-1}(u_\xi+1)
  \bigl(\norm{e_n}_{m,\xi}+1\bigr).
  \label{eq:lemma3-H-growth-euclidean}
\end{align}
Since $e_{n+1}=e_n+\alpha_nF(\theta_n,Y_n)$,
\begin{align}
  \norm{e_{n+1}}_{m,\xi}
  =& \norm{e_n+\alpha_nF(\theta_n,Y_n)}_{m,\xi}\\
  \leq& \norm{e_n}_{m,\xi}
  +\alpha_n\norm{F(\theta_n,Y_n)}_{m,\xi}\\
  \leq& \norm{e_n}_{m,\xi}
  +\alpha_nL_F\ell_\xi^{-1}(u_\xi+1)
  \bigl(\norm{e_n}_{m,\xi}+1\bigr)
  \explain{By \eqref{eq:lemma3-F-growth-moreau}}\\
  \leq& \left(1+\bar\alpha L_F\ell_\xi^{-1}(u_\xi+1)\right)
  \bigl(\norm{e_n}_{m,\xi}+1\bigr).
  \label{eq:lemma3-enplus-growth-explicit}
\end{align}
For every $x\in\R[d]$,
\begin{align}
  \norm{\nabla M_\xi(x)}_2
  =& \sup_{\norm{z}_2\leq1}
  \abs{\inner{\nabla M_\xi(x)}{z}}\\
  \leq& \sup_{\norm{z}_2\leq1}
  \norm{x}_{m,\xi}\norm{z}_{m,\xi}
  \explain{By \eqref{eq:moreau-gradient-cauchy}}\\
  \leq& u\ell_\xi^{-1}\norm{x}_{m,\xi}.
  \label{eq:lemma3-gradient-euclidean-growth}
\end{align}
Combining \eqref{eq:lemma3-gradient-euclidean-growth} with
\eqref{eq:lemma3-enplus-growth-explicit} gives
\begin{align}
  \norm{\nabla M_\xi(e_{n+1})}_2
  \leq& u\ell_\xi^{-1}
  \left(1+\bar\alpha L_F\ell_\xi^{-1}(u_\xi+1)\right)
  \bigl(\norm{e_n}_{m,\xi}+1\bigr).
  \label{eq:lemma3-gradient-next-growth}
\end{align}
The smoothness of $M_\xi$ and the Lipschitzness of $H_\theta$ also imply
\begin{align}
  \norm{\nabla M_\xi(e_{n+1})-\nabla M_\xi(e_n)}_2
  \leq& \frac{1}{\xi}\norm{e_{n+1}-e_n}_2
  \explain{By \eqref{eq:moreau-smoothness}}\\
  =& \frac{\alpha_n}{\xi}\norm{F(\theta_n,Y_n)}_2\\
  \leq& \frac{\alpha_nL_F(u_\xi+1)}{\xi\ell}
  \bigl(\norm{e_n}_{m,\xi}+1\bigr),
  \label{eq:lemma3-gradient-difference-explicit}\\
  \norm{H_{\theta_{n+1}}(y)-H_{\theta_n}(y)}_2
  \leq& \ell^{-1}\norm{H_{\theta_{n+1}}(y)-H_{\theta_n}(y)}\\
  \leq& L_H\ell^{-1}\norm{\theta_{n+1}-\theta_n}
  \explain{By Assumption~\ref{asp:poisson}}\\
  =& \alpha_nL_H\ell^{-1}\norm{F(\theta_n,Y_n)}\\
  \leq& \alpha_nL_FL_H\ell^{-1}(u_\xi+1)
  \bigl(\norm{e_n}_{m,\xi}+1\bigr).
  \label{eq:lemma3-H-difference-explicit}
\end{align}

We now use the smoothness of $M_\xi$. Since
$e_{n+1}-e_n=\alpha_nF(\theta_n,Y_n)$,
\begin{align}
  M_\xi(e_{n+1})
  \leq& M_\xi(e_n)
  +\alpha_n\inner{\nabla M_\xi(e_n)}{F(\theta_n,Y_n)}
  +\frac{\alpha_n^2}{2\xi}\norm{F(\theta_n,Y_n)}_2^2
  \explain{By \eqref{eq:moreau-smoothness}}\\
  \leq& M_\xi(e_n)
  +\alpha_n\inner{\nabla M_\xi(e_n)}{F(\theta_n,Y_n)}\\
  &+\frac{L_F^2(u_\xi+1)^2}{2\xi\ell^2}
  \alpha_n^2\bigl(\norm{e_n}_{m,\xi}+1\bigr)^2
  \explain{By \eqref{eq:lemma3-F-growth-euclidean}}\\
  \leq& M_\xi(e_n)
  +\alpha_n\inner{\nabla M_\xi(e_n)}{F(\theta_n,Y_n)}\\
  &+\frac{L_F^2(u_\xi+1)^2}{\xi\ell^2}
  \alpha_n^2\bigl(\norm{e_n}_{m,\xi}^2+1\bigr).
  \label{eq:lemma3-moreau-taylor-explicit}
\end{align}
The Poisson equation gives
\begin{align}
\label{eq:lemma3-poisson-split-explicit}
  \inner{\nabla M_\xi(e_n)}{F(\theta_n,Y_n)}
  =& \inner{\nabla M_\xi(e_n)}{f(\theta_n)}
  +\inner{\nabla M_\xi(e_n)}{H_{\theta_n}(Y_n)}\\
  &-\inner{\nabla M_\xi(e_n)}{(PH_{\theta_n})(Y_n)}
  \explain{By \eqref{eq:poisson}}.
\end{align}
Since $\nabla M_\xi(e_n)$, $\theta_n$, and $Y_n$ are $\fF_n$-measurable,
\begin{align}
\label{eq:lemma3-Markov-use-explicit}
  \inner{\nabla M_\xi(e_n)}{(PH_{\theta_n})(Y_n)}
  =& \E_n\qty[
    \inner{\nabla M_\xi(e_n)}{H_{\theta_n}(Y_{n+1})}
  ]
  \explain{By \eqref{eq:bg-markov-property}}.
\end{align}
Combining with \eqref{eq:lemma3-moreau-taylor-explicit} and taking conditional expectation, we get
\begin{align}
  \E_n\qty[M_\xi(e_{n+1})]
  \leq& M_\xi(e_n)
  +\alpha_n\inner{\nabla M_\xi(e_n)}{f(\theta_n)}\\
  &+\alpha_n\E_n\qty[
    \inner{\nabla M_\xi(e_n)}{H_{\theta_n}(Y_n)}
    -\inner{\nabla M_\xi(e_n)}{H_{\theta_n}(Y_{n+1})}
  ]\\
  &+\frac{L_F^2(u_\xi+1)^2}{\xi\ell^2}
  \alpha_n^2\bigl(\norm{e_n}_{m,\xi}^2+1\bigr).
  \label{eq:lemma3-before-next-time-explicit}
\end{align}

It remains to replace the one-step Poisson scalar by the genuine next-time
scalar. By \eqref{eq:lemma3-H-growth-euclidean},
\eqref{eq:lemma3-gradient-next-growth},
\eqref{eq:lemma3-gradient-difference-explicit}, and
\eqref{eq:lemma3-H-difference-explicit},
\begin{align}
&\abs{
  \inner{\nabla M_\xi(e_{n+1})}{H_{\theta_{n+1}}(Y_{n+1})}
  -\inner{\nabla M_\xi(e_n)}{H_{\theta_n}(Y_{n+1})}
}\\
\leq& \norm{\nabla M_\xi(e_{n+1})-\nabla M_\xi(e_n)}_2
\norm{H_{\theta_n}(Y_{n+1})}_2\\
&+\norm{\nabla M_\xi(e_{n+1})}_2
\norm{H_{\theta_{n+1}}(Y_{n+1})-H_{\theta_n}(Y_{n+1})}_2
\explain{Cauchy--Schwarz inequality}\\
\leq& \frac{\alpha_nL_FL_H(u_\xi+1)^2}{\xi\ell^2}
\bigl(\norm{e_n}_{m,\xi}+1\bigr)^2\\
&+\alpha_nL_FL_H
\frac{u(u_\xi+1)}{\ell_\xi\ell}
\left(1+\bar\alpha L_F\ell_\xi^{-1}(u_\xi+1)\right)
\bigl(\norm{e_n}_{m,\xi}+1\bigr)^2\\
\leq& 2\alpha_nL_FL_H\left[
\frac{(u_\xi+1)^2}{\xi\ell^2}
+\frac{u(u_\xi+1)}{\ell_\xi\ell}
\left(1+\bar\alpha L_F\ell_\xi^{-1}(u_\xi+1)\right)
\right]
\bigl(\norm{e_n}_{m,\xi}^2+1\bigr).
\label{eq:lemma3-next-time-replacement-explicit}
\end{align}
Therefore,
\begin{align}
&\alpha_n\E_n\qty[
  \inner{\nabla M_\xi(e_n)}{H_{\theta_n}(Y_n)}
  -\inner{\nabla M_\xi(e_n)}{H_{\theta_n}(Y_{n+1})}
]\\
=& \alpha_n\E_n\qty[
  \inner{\nabla M_\xi(e_n)}{H_{\theta_n}(Y_n)}
  -\inner{\nabla M_\xi(e_{n+1})}{H_{\theta_{n+1}}(Y_{n+1})}
]\\
&+\alpha_n\E_n\qty[
  \inner{\nabla M_\xi(e_{n+1})}{H_{\theta_{n+1}}(Y_{n+1})}
  -\inner{\nabla M_\xi(e_n)}{H_{\theta_n}(Y_{n+1})}
]\\
\leq& \alpha_n\E_n\qty[
  \inner{\nabla M_\xi(e_n)}{H_{\theta_n}(Y_n)}
  -\inner{\nabla M_\xi(e_{n+1})}{H_{\theta_{n+1}}(Y_{n+1})}
]\\
&+2\alpha_n^2L_FL_H\left[
\frac{(u_\xi+1)^2}{\xi\ell^2}
+\frac{u(u_\xi+1)}{\ell_\xi\ell}
\left(1+\bar\alpha L_F\ell_\xi^{-1}(u_\xi+1)\right)
\right]
\bigl(\norm{e_n}_{m,\xi}^2+1\bigr)
\explain{By \eqref{eq:lemma3-next-time-replacement-explicit}}.
\label{eq:lemma3-one-step-to-next-time-explicit}
\end{align}
Substituting into
\eqref{eq:lemma3-before-next-time-explicit} and denote
\begin{align}
    C_\xi \doteq\frac{L_F^2(u_\xi+1)^2}{\xi\ell^2}+2L_FL_H\left[
\frac{(u_\xi+1)^2}{\xi\ell^2}
+\frac{u(u_\xi+1)}{\ell_\xi\ell}
\left(1+\bar\alpha L_F\ell_\xi^{-1}(u_\xi+1)\right)
\right],
\end{align}
we then obtain
\begin{align}
\E_n \qty[M_\xi(e_{n+1})] \leq& M_\xi(e_n)
+\alpha_n\inner{\nabla M_\xi(e_n)}{f(\theta_n)}\\
&+\alpha_n \E_n\qty[
\inner{\nabla M_\xi(e_n)}{H_{\theta_n}(Y_n)}
-\inner{\nabla M_\xi(e_{n+1})}{H_{\theta_{n+1}}(Y_{n+1})}]\\
&+C_\xi\alpha_n^2\bigl(\norm{e_n}_{m,\xi}^2+1\bigr),
\end{align}
which proves the claim.
\end{proof}

\subsection{Proof of Lemma~\ref{lem:moreau-shifted-drift}}
\label{proof:moreau-shifted-drift}

\begin{proof}
Choose $K_\xi$ as in Lemma~\ref{lem:shifted-energy-coercive}, and fix
$K\geq K_\xi$. Let $C_{\xi,K}'$ be the constant in
\eqref{eq:appendix-V-M-comparison}. Choose a deterministic tail
$N_{\xi,K}$ such that, for every $n\geq N_{\xi,K}$,
\eqref{eq:appendix-coercive}, \eqref{eq:appendix-V-M-comparison}, and
$\alpha_n\leq1$ all hold.

The lower bound \eqref{eq:moreau-V-coercive} follows directly from
\eqref{eq:appendix-coercive}. In particular, on this deterministic tail,
\begin{align}
\label{eq:lemma4-coercive-use}
  \norm{e_n}_{m,\xi}^2
  \leq& 4V_n^\xi,
  \quad
  V_n^\xi\geq0.
\end{align}

It remains to prove \eqref{eq:moreau-V-drift}. Recall we have $\norm{x}_{m,\xi} \leq \ell_\xi^{-1}\norm{x}$ and $\norm{x}
\leq u_\xi\norm{x}_{m,\xi}$.
For every $n\geq N_{\xi,K}$, we then derive
\begin{align}
  \norm{F(\theta_n,Y_n)}_{m,\xi}
  \leq& \ell_\xi^{-1}\norm{F(\theta_n,Y_n)}\\
  \leq& \ell_\xi^{-1}
  \Bigl(
    \norm{F(\theta_n,Y_n)-F(\theta_*,Y_n)}
    +\norm{F(\theta_*,Y_n)}
  \Bigr)\\
  \leq& \ell_\xi^{-1}L_F\bigl(\norm{e_n}+1\bigr)
  \explain{By Assumption~\ref{asp:lip}}\\
  \leq& L_F\ell_\xi^{-1}(u_\xi+1)
  \bigl(\norm{e_n}_{m,\xi}+1\bigr).
  \label{eq:lemma4-F-growth}
\end{align}
Since $e_{n+1}=e_n+\alpha_nF(\theta_n,Y_n)$ and $\alpha_n\leq1$ on the
tail,
\begin{align}
  \norm{e_{n+1}}_{m,\xi}
  \leq& \norm{e_n}_{m,\xi}
  +\alpha_n\norm{F(\theta_n,Y_n)}_{m,\xi}\\
  \leq& \Bigl(1+L_F\ell_\xi^{-1}(u_\xi+1)\Bigr)
  \bigl(\norm{e_n}_{m,\xi}+1\bigr).
  \label{eq:lemma4-enplus-growth}
\end{align}
Similarly, we have
\begin{align}
  \norm{H_{\theta_{n+1}}(Y_{n+1})}_{m,\xi}
  \leq& \ell_\xi^{-1}\norm{H_{\theta_{n+1}}(Y_{n+1})}\\
  \leq& \ell_\xi^{-1}L_H\bigl(\norm{e_{n+1}}+1\bigr)
  \explain{By Assumption~\ref{asp:poisson}}\\
  \leq& L_H\ell_\xi^{-1}(u_\xi+1)
  \bigl(\norm{e_{n+1}}_{m,\xi}+1\bigr)\\
  \leq& L_H\ell_\xi^{-1}(u_\xi+1)
  \Bigl(2+L_F\ell_\xi^{-1}(u_\xi+1)\Bigr)
  \bigl(\norm{e_n}_{m,\xi}+1\bigr).
  \label{eq:lemma4-H-next-growth}
\end{align}
Combining \eqref{eq:lemma4-enplus-growth} and \eqref{eq:lemma4-H-next-growth}, we then get
\begin{align}
  \abs{
    \inner{\nabla M_\xi(e_{n+1})}
    {H_{\theta_{n+1}}(Y_{n+1})}
  }
  \leq& \norm{e_{n+1}}_{m,\xi}
  \norm{H_{\theta_{n+1}}(Y_{n+1})}_{m,\xi}
  \explain{By \eqref{eq:moreau-gradient-cauchy}}\\
  \leq& C_\xi'
  \bigl(\norm{e_n}_{m,\xi}^2+1\bigr),
  \label{eq:lemma4-next-poisson-bound}
\end{align}
where $C_\xi'
  \doteq 2L_H\ell_\xi^{-1}(u_\xi+1)
  \Bigl(1+L_F\ell_\xi^{-1}(u_\xi+1)\Bigr)
  \Bigl(2+L_F\ell_\xi^{-1}(u_\xi+1)\Bigr)$.

Now apply Lemma~\ref{lem:conditional-chandak-moreau-poisson}. Since
\(\inner{\nabla M_\xi(e_n)}{H_{\theta_n}(Y_n)}\) is $\fF_n$-measurable,
\begin{align}
  \E_n\bigl[V_{n+1}^\xi\bigr]
  =& \E_n\bigl[M_\xi(e_{n+1})\bigr]
  +\alpha_{n+1}
  \E_n\bigl[
    \inner{\nabla M_\xi(e_{n+1})}
    {H_{\theta_{n+1}}(Y_{n+1})}
  \bigr]
  +K\alpha_{n+1}^2\\
  \leq& M_\xi(e_n)
  +\alpha_n\inner{\nabla M_\xi(e_n)}{f(\theta_n)}
  +\alpha_n\inner{\nabla M_\xi(e_n)}{H_{\theta_n}(Y_n)}\\
  &-\alpha_n
  \E_n\bigl[
    \inner{\nabla M_\xi(e_{n+1})}
    {H_{\theta_{n+1}}(Y_{n+1})}
  \bigr]+C_\xi
  \alpha_n^2\bigl(\norm{e_n}_{m,\xi}^2+1\bigr)\\
  &+\alpha_{n+1}
  \E_n\bigl[
    \inner{\nabla M_\xi(e_{n+1})}
    {H_{\theta_{n+1}}(Y_{n+1})}
  \bigr]
  +K\alpha_{n+1}^2
  \explain{By Lemma~\ref{lem:conditional-chandak-moreau-poisson}}\\
  =& V_n^\xi
  +\alpha_n\inner{\nabla M_\xi(e_n)}{f(\theta_n)}+\bigl(\alpha_{n+1}-\alpha_n\bigr)
  \E_n\bigl[
    \inner{\nabla M_\xi(e_{n+1})}
    {H_{\theta_{n+1}}(Y_{n+1})}
  \bigr]\\
  &+C_\xi
  \alpha_n^2\bigl(\norm{e_n}_{m,\xi}^2+1\bigr)
  +K\bigl(\alpha_{n+1}^2-\alpha_n^2\bigr)
  \explain{By \eqref{eq:moreau-V-def}}\\
  \leq& V_n^\xi
  +\alpha_n\inner{\nabla M_\xi(e_n)}{f(\theta_n)}+\abs{\alpha_{n+1}-\alpha_n}
  \E_n\bigl[
    \abs{
      \inner{\nabla M_\xi(e_{n+1})}
      {H_{\theta_{n+1}}(Y_{n+1})}
    }
  \bigr]\\
  &+C_\xi
  \alpha_n^2\bigl(\norm{e_n}_{m,\xi}^2+1\bigr),
  \label{eq:lemma4-after-lemma3}
\end{align}
where the last step uses that $\alpha_n$ is nonincreasing and $K\geq0$.
Taking conditional expectation in \eqref{eq:lemma4-next-poisson-bound} and
substituting into \eqref{eq:lemma4-after-lemma3} yields
\begin{align}
  \E_n\bigl[V_{n+1}^\xi\bigr]
  \leq& V_n^\xi
  +\alpha_n\inner{\nabla M_\xi(e_n)}{f(\theta_n)}+\Bigl(
    C_\xi
    +C_\xi'
  \Bigr)
  r_n\bigl(\norm{e_n}_{m,\xi}^2+1\bigr),
  \label{eq:lemma4-before-contract}
\end{align}
where $r_n=\alpha_n^2+\abs{\alpha_{n+1}-\alpha_n}$.

By Lemma~\ref{lem:moreau-contract-drift},
\begin{align}
  \inner{\nabla M_\xi(e_n)}{f(\theta_n)}
  \leq& -\mu_\xi M_\xi(e_n).
  \label{eq:lemma4-contract-use}
\end{align}
Combining \eqref{eq:lemma4-before-contract} and
\eqref{eq:lemma4-contract-use}, we obtain
\begin{align}
  \E_n\bigl[V_{n+1}^\xi\bigr]
  \leq& V_n^\xi
  -\mu_\xi\alpha_nM_\xi(e_n)+\Bigl(
    C_\xi
    +C_\xi'
  \Bigr)
  r_n\bigl(\norm{e_n}_{m,\xi}^2+1\bigr).
  \label{eq:lemma4-before-M-to-V}
\end{align}
The comparison bound \eqref{eq:appendix-V-M-comparison} gives
\begin{align}
  -\mu_\xi\alpha_nM_\xi(e_n)
  =& -\mu_\xi\alpha_nV_n^\xi
  +\mu_\xi\alpha_n\bigl(V_n^\xi-M_\xi(e_n)\bigr)\\
  \leq& -\mu_\xi\alpha_nV_n^\xi
  +\mu_\xi C_{\xi,K}'
  r_n\bigl(\norm{e_n}_{m,\xi}^2+1\bigr).
  \label{eq:lemma4-M-to-V}
\end{align}
Substituting into
\eqref{eq:lemma4-before-M-to-V} then gives
\begin{align}
  \E_n\bigl[V_{n+1}^\xi\bigr]
  \leq& \bigl(1-\mu_\xi\alpha_n\bigr)V_n^\xi
  +(C_\xi+C_\xi'+\mu_\xi C_{\xi,K}')r_n
  \bigl(\norm{e_n}_{m,\xi}^2+1\bigr)\\
  \leq& \bigl(1-\mu_\xi\alpha_n\bigr)V_n^\xi
  +(C_\xi+C_\xi'+\mu_\xi C_{\xi,K}')r_n
  \bigl(4V_n^\xi+1\bigr).
  \explain{By \eqref{eq:lemma4-coercive-use}}
\end{align}
Let $C_{\xi,K}\doteq 4(C_\xi+C_\xi'+\mu_\xi C_{\xi,K}')$ then completes the proof.
\end{proof}

\subsection{Proof of Theorem~\ref{thm:moreau-power-rates}}
\label{proof:moreau-power-rates}

\begin{proof}
Fix $\zeta$ in the range claimed by the theorem. If $\frac{1}{2}<\eta<1$,
choose any $\xi>0$ satisfying \eqref{eq:xi-contraction-choice}. If
$\eta=1$, then $\zeta<\min\{1,2(1-\kappa)\alpha\}$. Since
\begin{align}
  \mu_\xi
  =& 2\left(1-\kappa\frac{u_\xi}{\ell_\xi}\right)
  \longrightarrow 2(1-\kappa)
  \quad \text{as } \xi\downarrow0,
\end{align}
we may choose $\xi>0$ small enough such that
\begin{align}
\label{eq:theorem-proof-xi-choice}
  \kappa\frac{u_\xi}{\ell_\xi}
  &<1,
  \quad
  \zeta<\min\{1,\mu_\xi\alpha\}.
\end{align}

Let $K_\xi$ be the constant from
Lemma~\ref{lem:moreau-shifted-drift}, and fix $K\geq K_\xi$. Then there
exist constants $C_{\xi,K}<\infty$ and $N_{\xi,K}<\infty$ such that, for
every $n\geq N_{\xi,K}$,
\begin{align}
\label{eq:theorem-proof-V-coercive}
  V_n^\xi
  \geq& \frac{1}{4}\norm{e_n}_{m,\xi}^2
  \geq 0,\\
\label{eq:theorem-proof-V-drift}
  \E_n\bigl[V_{n+1}^\xi\bigr]
  \leq& \bigl(1-\mu_\xi\alpha_n+C_{\xi,K}r_n\bigr)V_n^\xi
  +C_{\xi,K}r_n .
\end{align}
Thus Lemma~\ref{lem:power-law-drift-rates} applies on this deterministic
tail with $X_n=V_n^\xi$, $c_0=\mu_\xi$, and $C_0=C_{\xi,K}$. Hence
\begin{align}
\label{eq:theorem-proof-V-rate}
  (n+1)^\zeta V_n^\xi
  \longrightarrow& 0
  \quad \text{a.s.}
\end{align}
Indeed, when $\frac{1}{2}<\eta<1$, this follows from
$\zeta<2\eta-1$. When $\eta=1$, this follows from
\eqref{eq:theorem-proof-xi-choice}.

It remains to translate the rate for $V_n^\xi$ into the rate for the
original error. On the same deterministic tail,
\begin{align}
  (n+1)^\zeta\norm{\theta_n-\theta_*}^2
  = (n+1)^\zeta\norm{e_n}^2
  \leq u_\xi^2(n+1)^\zeta\norm{e_n}_{m,\xi}^2
  \leq 4u_\xi^2(n+1)^\zeta V_n^\xi.
\end{align}
Thus $(n+1)^\zeta\norm{\theta_n-\theta_*}^2 \rightarrow 0$ almost surely.
Since changing finitely many initial terms does not affect the almost
sure limit, the claimed convergence holds for the full sequence.
\end{proof}

\subsection{Proof of Corollary~\ref{cor:ms-rates}}
\label{proof:ms-rates}
\begin{proof}
The proof of Lemma~\ref{lem:moreau-shifted-drift} does not use square summability of the learning rates.
Thus, for any fixed $0<\eta\leq1$, the same argument gives the following deterministic-tail bounds.
Fix $K\geq K_\xi$. There exist constants $C_{\xi,K}<\infty$ and a deterministic index $N_1$ such that, for every $n\geq N_1$,
\begin{align}
\label{eq:appendix-ms-coercive}
  V_n^\xi
  \geq& \frac{1}{4}\norm{e_n}_{m,\xi}^2
  \geq 0,\\
\label{eq:appendix-ms-drift}
  \E_n\qty[V_{n+1}^\xi]
  \leq& \bigl(1-\mu_\xi\alpha_n+C_{\xi,K}r_n\bigr)V_n^\xi
  +C_{\xi,K}r_n.
\end{align}

We first handle the finite initial segment before this deterministic tail.
For every $m\geq0$,
\begin{align}
  \norm{e_{m+1}}+1
  \leq& \norm{e_m}+1+\alpha_m\norm{F(\theta_m,Y_m)}\\
  \leq& \norm{e_m}+1+L_F\alpha_m\bigl(\norm{e_m}+1\bigr)
  \explain{By Assumption~\ref{asp:lip}}\\
  =& \bigl(1+L_F\alpha_m\bigr)\bigl(\norm{e_m}+1\bigr).
\end{align}
Iterating to any fixed time $m$ gives
\begin{align}
\label{eq:appendix-ms-finite-initial-e}
  \norm{e_m}+1
  \leq& \bigl(\norm{e_0}+1\bigr)
  \prod_{i=0}^{m-1}\bigl(1+L_F\alpha_i\bigr)
  <\infty .
\end{align}
Hence, by \eqref{eq:moreau-gradient-cauchy}, Assumption~\ref{asp:poisson}, and norm comparison,
\begin{align}
  \abs{V_m^\xi}
  \leq& \frac{1}{2}\norm{e_m}_{m,\xi}^2
  +\alpha_m\norm{e_m}_{m,\xi}
  \norm{H_{\theta_m}(Y_m)}_{m,\xi}
  +K\alpha_m^2\\
  \leq& \frac{1}{2}\ell_\xi^{-2}\norm{e_m}^2
  +L_H\ell_\xi^{-2}\alpha_m\norm{e_m}
  \bigl(\norm{e_m}+1\bigr)
  +K\alpha_m^2
  <\infty .
  \label{eq:appendix-ms-finite-initial-V}
\end{align}
Consequently, for every deterministic index $m$, $\E\qty[\abs{V_m^\xi}]<\infty$.

For $0<\eta\leq1$, the mean-value theorem gives
\begin{align}
  \abs{\alpha_{n+1}-\alpha_n}
  =& \alpha\Bigl((n+1)^{-\eta}-(n+2)^{-\eta}\Bigr)\\
  \leq& \alpha\eta(n+1)^{-\eta-1}\\
  \leq& \frac{\eta}{\alpha}\alpha_n^2,
\end{align}
where the last step uses $\eta\leq1$. Therefore, $r_n \leq C_r\alpha_n^2$, where $C_r\doteq 1+\frac{\eta}{\alpha}$.
Taking total expectation in \eqref{eq:appendix-ms-drift} then gives, for every $n\geq N_1$,
\begin{align}
\label{eq:appendix-ms-recursion}
  \E\qty[V_{n+1}^\xi]
  \leq& \bigl(1-\mu_\xi\alpha_n+C_{\xi,K}C_r\alpha_n^2\bigr)
  \E\qty[V_n^\xi]
  +C_{\xi,K}C_r\alpha_n^2 .
\end{align}

We first consider $0<\eta<1$. Since $\alpha_n=\alpha(n+1)^{-\eta}$, choose a deterministic index $N_2$ such that
\begin{align}
  C_{\xi,K}C_r\alpha_n^2
  \leq& \frac{\mu_\xi\alpha}{2}(n+1)^{-\eta}
\end{align}
for every $n\geq N_2$. Hence \eqref{eq:appendix-ms-recursion} implies
\begin{align}
\label{eq:appendix-ms-recursion-subharmonic}
  \E\qty[V_{n+1}^\xi]
  \leq& \left(1-\frac{\mu_\xi\alpha}{2(n+1)^\eta}\right)
  \E\qty[V_n^\xi]
  +\frac{C_{\xi,K}C_r\alpha^2}{(n+1)^{2\eta}}
\end{align}
for every $n\geq \max\qty{N_1,N_2}$.
Choose a deterministic index $N_3$ such that, for every $n\geq N_3$,
\begin{align}
\label{eq:appendix-ms-power-comparison}
  (n+2)^{-\eta}
  \geq& (n+1)^{-\eta}
  -\frac{\mu_\xi\alpha}{4}(n+1)^{-2\eta}.
\end{align}
This is possible because $\eta<1$. Set
\begin{align}
  N
  \doteq& \max\qty{N_1,N_2,N_3}.
\end{align}
Choose $B<\infty$ large enough such that
\begin{align}
  \E\qty[V_N^\xi]
  \leq& B(N+1)^{-\eta},
  \quad
  \frac{B\mu_\xi\alpha}{4}
  \geq C_{\xi,K}C_r\alpha^2 .
\end{align}
Then induction gives, for every $n\geq N$,
\begin{align}
\label{eq:appendix-ms-subharmonic-rate}
  \E\qty[V_n^\xi]
  \leq& B(n+1)^{-\eta}.
\end{align}
Indeed, if the bound holds at time $n$, then
\begin{align}
  \E\qty[V_{n+1}^\xi]
  \leq& B(n+1)^{-\eta}
  -\frac{B\mu_\xi\alpha}{2}(n+1)^{-2\eta}
  +C_{\xi,K}C_r\alpha^2(n+1)^{-2\eta}
  \explain{By \eqref{eq:appendix-ms-recursion-subharmonic}}\\
  \leq& B(n+1)^{-\eta}
  -\frac{B\mu_\xi\alpha}{4}(n+1)^{-2\eta}\\
  \leq& B(n+2)^{-\eta}
  \explain{By \eqref{eq:appendix-ms-power-comparison}}.
\end{align}
Thus $\E\qty[V_n^\xi]=\fO(n^{-\eta})$ when $0<\eta<1$.

Now consider $\eta=1$. In this case \eqref{eq:appendix-ms-recursion} gives, for every $n\geq N_1$,
\begin{align}
\label{eq:appendix-ms-recursion-harmonic}
  \E\qty[V_{n+1}^\xi]
  \leq& \left(
    1-\frac{\mu_\xi\alpha}{n+1}
    +\frac{C_{\xi,K}C_r\alpha^2}{(n+1)^2}
  \right)\E\qty[V_n^\xi]
  +\frac{C_{\xi,K}C_r\alpha^2}{(n+1)^2}.
\end{align}
Choose a deterministic index $N_4$ such that, for every $j\geq N_4$,
\begin{align}
  1-\frac{\mu_\xi\alpha}{j+1}
  +\frac{C_{\xi,K}C_r\alpha^2}{(j+1)^2}
  \geq& 0.
\end{align}
Set
\begin{align}
  N
  \doteq& \max\qty{N_1,N_4}.
\end{align}
For $n>i\geq N-1$,
\begin{align}
\label{eq:appendix-product-bound}
  &\prod_{j=i+1}^{n-1}
  \left(
    1-\frac{\mu_\xi\alpha}{j+1}
    +\frac{C_{\xi,K}C_r\alpha^2}{(j+1)^2}
  \right)\\
  \leq& \exp\left(
    -\mu_\xi\alpha\sum_{j=i+1}^{n-1}\frac{1}{j+1}
    +C_{\xi,K}C_r\alpha^2
    \sum_{j=i+1}^{n-1}\frac{1}{(j+1)^2}
  \right)\\
  \leq& \exp\left(
    C_{\xi,K}C_r\alpha^2
    \sum_{j=0}^{\infty}\frac{1}{(j+1)^2}
  \right)
  \left(\frac{i+2}{n+1}\right)^{\mu_\xi\alpha}\\
  \leq& \exp\left(2C_{\xi,K}C_r\alpha^2\right)
  \left(\frac{i+2}{n+1}\right)^{\mu_\xi\alpha}.
\end{align}
Iterating \eqref{eq:appendix-ms-recursion-harmonic} from the deterministic time $N$ yields
\begin{align}
  \E\qty[V_n^\xi]
  \leq& \E\qty[V_N^\xi]
  \prod_{j=N}^{n-1}
  \left(
    1-\frac{\mu_\xi\alpha}{j+1}
    +\frac{C_{\xi,K}C_r\alpha^2}{(j+1)^2}
  \right)\\
  &+C_{\xi,K}C_r\alpha^2
  \sum_{i=N}^{n-1}\frac{1}{(i+1)^2}
  \prod_{j=i+1}^{n-1}
  \left(
    1-\frac{\mu_\xi\alpha}{j+1}
    +\frac{C_{\xi,K}C_r\alpha^2}{(j+1)^2}
  \right)\\
  \leq& \exp\left(2C_{\xi,K}C_r\alpha^2\right)\E\qty[V_N^\xi]
  \left(\frac{N+1}{n+1}\right)^{\mu_\xi\alpha}\\
  &+C_{\xi,K}C_r\alpha^2\exp\left(2C_{\xi,K}C_r\alpha^2\right)
  \sum_{i=N}^{n-1}\frac{1}{(i+1)^2}
  \left(\frac{i+2}{n+1}\right)^{\mu_\xi\alpha}.
  \label{eq:appendix-ms-harmonic-iterate}
\end{align}
The sum in \eqref{eq:appendix-ms-harmonic-iterate} is
$\fO(n^{-\mu_\xi\alpha})$ when $\mu_\xi\alpha<1$,
$\fO(\log n/n)$ when $\mu_\xi\alpha=1$, and
$\fO(n^{-1})$ when $\mu_\xi\alpha>1$. Therefore,
\begin{align}
\label{eq:appendix-ms-harmonic-rate}
  \E\qty[V_n^\xi]
  =&
  \begin{cases}
    \fO(n^{-\min\{\mu_\xi\alpha,1\}}),
    & \text{if } \mu_\xi\alpha\neq1,\\
    \fO(\log n/n),
    & \text{if } \mu_\xi\alpha=1.
  \end{cases}
\end{align}

Finally, for every $n\geq N$, \eqref{eq:appendix-ms-coercive} and
\eqref{eq:m-star-equivalence} give
\begin{align}
  \norm{\theta_n-\theta_*}^2
  =& \norm{e_n}^2 \leq u_\xi^2\norm{e_n}_{m,\xi}^2 \leq 4u_\xi^2V_n^\xi.
\end{align}
Taking expectations and using \eqref{eq:appendix-ms-subharmonic-rate} and
\eqref{eq:appendix-ms-harmonic-rate} gives the stated mean-square rates on the deterministic tail.
The finitely many indices $n<N$ have finite second moments by
\eqref{eq:appendix-ms-finite-initial-e}, so they do not affect the asymptotic $\fO(\cdot)$ rates.
\end{proof}

\section{Proof of Proposition~\ref{prop:gtd-assumptions}}
\label{proof:gtd-assumptions}
\begin{proof}
First, Assumption~\ref{asp:invariant} holds.  The process $Y_n=(S_n,A_n,\ldots,S_{n+N},A_{n+N})$ is a finite-state Markov chain induced by the behavior state-action chain.  Since the latter is irreducible and aperiodic, the window chain has a unique invariant distribution $\varpi$ on its recurrent class.  All quantities in \eqref{eq:gtd-A}--\eqref{eq:gtd-b} are bounded on this finite state space, so $F_{\rm GTD}(w,\cdot)$ is $\varpi$-integrable for every $w$.

Second, Assumption~\ref{asp:ctct} holds by Condition~\ref{cond:gtd-hurwitz} and the construction of $T_{\rm GTD}$.  Indeed,
\begin{align}
    f_{\rm GTD}(w)
    \doteq
    \E_{\varpi}[F_{\rm GTD}(w,Y)]
    =
    \beta_{c,\rho}\bigl(\bar A_{c,\rho}w-\bar b_{c,\rho}\bigr),
\end{align}
and hence $w+f_{\rm GTD}(w)=T_{\rm GTD}(w)$, which is a $\kappa_{\rm GTD}$-contraction in $\norm{\cdot}_{\rm GTD}$ by \eqref{eq:gtd-contract-ineq}.

Third, Assumption~\ref{asp:lip} holds.  Because $N$ is fixed, rewards and features are bounded, and $c,\rho$ are bounded, the matrices $A_{c,\rho}(y)$ and vectors $b_{c,\rho}(y)$ are uniformly bounded over all windows $y$.  Therefore $F_{\rm GTD}$ is affine in $w$ with uniformly bounded coefficients.  In particular, after changing constants to account for norm equivalence in finite dimension, there is $L<\infty$ such that for all $w,w'\in\mathbb{R}^d$ and all windows $y$,
\begin{align}
    \norm{F_{\rm GTD}(w,y)-F_{\rm GTD}(w',y)}_{\rm GTD}
    &\leq
    L\norm{w-w'}_{\rm GTD},\\
    \norm{F_{\rm GTD}(w^*_{c,\rho},y)}_{\rm GTD}
    &\leq
    L .
\end{align}

Finally, Assumption~\ref{asp:poisson} holds by the finite-state Poisson-equation argument.  Let $P_Y$ be the transition matrix of the window chain.  For each fixed $w$, the centered function $F_{\rm GTD}(w,\cdot)-f_{\rm GTD}(w)$ has mean zero under $\varpi$, and the Poisson equation admits the solution
\begin{align}
\label{eq:gtd-poisson-solution}
    H_w(y)
    \doteq
    \sum_{t=0}^{\infty}
    P_Y^t\bigl(F_{\rm GTD}(w,\cdot)-f_{\rm GTD}(w)\bigr)(y),
\end{align}
where the series converges componentwise on the recurrent class by finite-state geometric ergodicity.  Since $F_{\rm GTD}$ is affine in $w$ with uniformly bounded coefficients, and the linear operator in \eqref{eq:gtd-poisson-solution} is bounded on centered functions over the finite state space, the Poisson solution inherits the required Lipschitz and growth bounds:
\begin{align}
    \norm{H_w(y)-H_{w'}(y)}_{\rm GTD}
    \leq L_H\norm{w-w'}_{\rm GTD},
    \quad
    \norm{H_{w^*_{c,\rho}}(y)}_{\rm GTD}\leq L_H
\end{align}
for a finite constant $L_H$.  This verifies Assumption~\ref{asp:poisson}.  The identification $\alpha_n=a_n/\beta_{c,\rho}$ in \eqref{eq:gtd-SA-identification} gives the learning-rate form required by Assumption~\ref{asp:lrn_rt}, with stepsize constant $a/\beta_{c,\rho}$.
\end{proof}

  \clearpage

\section{Numerical Experiment}
\label{app:revision-experiment}
As \citet{lauand2024revisiting} provide numerical experiments for linear Markovian SA with power-law stepsizes $\alpha_n\propto n^{-\eta}$, $0<\eta<1$, we focus here on the harmonic case $\eta=1$.

We run linear TD(0) on a 100-state nonreversible Markov reward process with discount $\gamma=0.7$. Let $\Pi_1,\ldots,\Pi_4$ be independent uniformly sampled permutation matrices (environment seed 123), and set
\[
P=0.05J+0.95(0.1I+0.9Q),
\qquad Q=\tfrac14\sum_{j=1}^4\Pi_j,
\qquad J=\tfrac1{100}\mathbf 1\mathbf 1^\top.
\]
Then $D=I/100$ is stationary and $\norm{P-P^\top}_F=6.02$. The seven features are the constant feature and the sine/cosine Fourier modes of frequencies $k=1,2,3$, with each nonconstant feature multiplied by $\sqrt2$, so $\Phi^\top D\Phi=I$ up to machine precision. For $s=0,\ldots,99$, rewards are
\[
r(s)=0.4+0.35\sin u_s+0.2\cos(2u_s)+0.05\sin(3u_s),
\qquad u_s=2\pi s/100.
\]
With $A=\Phi^\top D(\gamma P-I)\Phi$, we compute the exact TD root $w^\star=-A^{-1}\Phi^\top Dr$ and solve $A^\top G+GA=-I$. The contraction scaling $\beta=0.1$ gives $\kappa=0.97$.

For each $a\in\{1/3,2/3,1,4/3,1.6\}$, we run 30 independent single trajectories from uniform initial states and $w_0=0$, using stepsizes $a/(n+1)$ at updates $n=1,\ldots,10^6$. NumPy's default generator uses seeds 2000--2004 for these five cases, respectively. We record 400 logarithmically spaced points in $[10^2,10^6]$ and measure
\[
E_n=(w_n-w^\star)^\top G(w_n-w^\star),
\qquad
\zeta_{\max}=\min\left\{1,\frac{2(1-\kappa)a}{\beta}\right\}.
\]
Theorem~\ref{thm:moreau-power-rates} permits every $\zeta<\zeta_{\max}$. Each fitted exponent $\widehat\zeta$ is the negative least-squares slope of $\log E_n$ against $\log n$ at the recorded points in $[10^5,10^6]$. Figure~\ref{fig:revision-exponent} shows the error trajectories and the fitted exponents, whose medians track the admissible endpoint as $a$ varies.

\begin{figure}[htbp]
\centering
\IfFileExists{figures/td0_harmonic_rate.pdf}{\includegraphics[width=\linewidth]{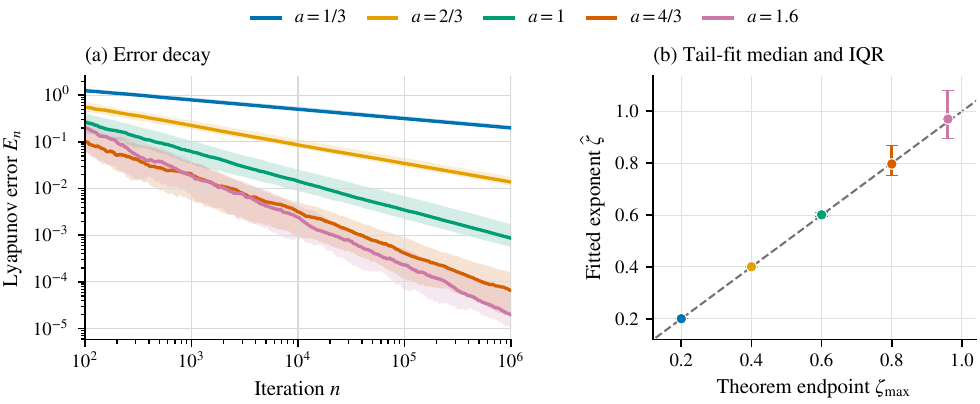}}{\includegraphics[width=\linewidth]{td0_harmonic_rate.pdf}}
\caption{Harmonic-learning-rate experiment, with 30 paths per stepsize. (a) Solid curves and shading show pointwise medians and interquartile ranges of $E_n$. (b) Points and error bars show medians and interquartile ranges of the individually fitted exponents over $[10^5,10^6]$; the dashed line is the identity reference.}
\label{fig:revision-exponent}
\end{figure}

\clearpage
 \section*{NeurIPS Paper Checklist}

\begin{enumerate}

\item {\bf Claims}
    \item[] Question: Do the main claims made in the abstract and introduction accurately reflect the paper's contributions and scope?
    \item[] Answer: \answerYes{} % Replace by \answerYes{}, \answerNo{}, or \answerNA{}.
    \item[] Justification: The abstract and introduction state that the paper studies almost sure convergence rates for stochastic approximation with Markovian noise under contractive expected updates, including possibly non-Euclidean contraction norms. The main technical and rate claims are formalized in Theorem 1 under Assumptions 2.1--2.5, and the reinforcement-learning application claims are formalized in Proposition 1 and Theorem 2. The application section also clarifies that the generalized TD recursion and its Hurwitz structural conditions are not new contributions; the contribution is to combine these structural properties with the proposed Poisson--Moreau drift to obtain pathwise almost sure rates under single-trajectory Markovian sampling.

\item {\bf Limitations}
    \item[] Question: Does the paper discuss the limitations of the work performed by the authors?
    \item[] Answer: \answerYes{} % Replace by \answerYes{}, \answerNo{}, or \answerNA{}.
    \item[] Justification: Sections 2 and 4 state the assumptions limiting the scope of the results. The conclusion identifies extending the almost-sure learning-rate range, attaining sharp LIL-order rates, and deriving concentration bounds as open problems.

\item {\bf Theory assumptions and proofs}
    \item[] Question: For each theoretical result, does the paper provide the full set of assumptions and a complete (and correct) proof?
    \item[] Answer: \answerYes{} % Replace by \answerYes{}, \answerNo{}, or \answerNA{}.
    \item[] Justification: The paper explicitly states the assumptions used for the general stochastic approximation result in Assumptions 2.1--2.5, and Theorem 1 is proved in Appendix B.4 using the supporting lemmas developed in Section 3 and Appendix B. The generalized TD application states its algorithm-specific assumptions in Section 4, verifies the abstract assumptions in Proposition 1 with proof in Appendix C, and then derives Theorem 2 by applying Theorem 1. Auxiliary results used in the proof, including the Moreau-envelope facts, the Poisson--Moreau drift recursion, and the almost-supermartingale rate lemma, are stated and proved or properly referenced in the appendix.
    

    \item {\bf Experimental result reproducibility}
    \item[] Question: Does the paper fully disclose all the information needed to reproduce the main experimental results of the paper to the extent that it affects the main claims and/or conclusions of the paper (regardless of whether the code and data are provided or not)?
    \item[] Answer: \answerYes{} % Replace by \answerYes{}, \answerNo{}, or \answerNA{}.
    \item[] Justification: Appendix~\ref{app:revision-experiment} specifies the synthetic Markov process, features, rewards, initialization, learning rates, random seeds, trajectory lengths, and exponent-fitting procedure.

\item {\bf Open access to data and code}
    \item[] Question: Does the paper provide open access to the data and code, with sufficient instructions to faithfully reproduce the main experimental results, as described in supplemental material?
    \item[] Answer: \answerNo{} % Replace by \answerYes{}, \answerNo{}, or \answerNA{}.
    \item[] Justification: The experimental code and recorded trajectories are not released. Appendix~\ref{app:revision-experiment} provides the synthetic environment construction and the simulation and evaluation protocol.

\item {\bf Experimental setting/details}
    \item[] Question: Does the paper specify all the training and test details (e.g., data splits, hyperparameters, how they were chosen, type of optimizer) necessary to understand the results?
    \item[] Answer: \answerYes{} % Replace by \answerYes{}, \answerNo{}, or \answerNA{}.
    \item[] Justification: Appendix~\ref{app:revision-experiment} provides the stepsize constants, number and length of trajectories, recording grid, and regression window used for the numerical illustration.

\item {\bf Experiment statistical significance}
    \item[] Question: Does the paper report error bars suitably and correctly defined or other appropriate information about the statistical significance of the experiments?
    \item[] Answer: \answerYes{} % Replace by \answerYes{}, \answerNo{}, or \answerNA{}.
    \item[] Justification: Figure~\ref{fig:revision-exponent} reports medians and empirical 25th--75th percentile intervals over 30 independent trajectories per stepsize, capturing variability from the sampled Markovian paths.

\item {\bf Experiments compute resources}
    \item[] Question: For each experiment, does the paper provide sufficient information on the computer resources (type of compute workers, memory, time of execution) needed to reproduce the experiments?
    \item[] Answer: \answerNo{} % Replace by \answerYes{}, \answerNo{}, or \answerNA{}.
    \item[] Justification: Appendix~\ref{app:revision-experiment} specifies the simulation scale, but hardware specifications and execution times are not reported.
    
\item {\bf Code of ethics}
    \item[] Question: Does the research conducted in the paper conform, in every respect, with the NeurIPS Code of Ethics \url{https://neurips.cc/public/EthicsGuidelines}?
    \item[] Answer: \answerYes{} % Replace by \answerYes{}, \answerNo{}, or \answerNA{}.
    \item[] Justification: This work consists of theoretical convergence analysis and a synthetic numerical illustration; it involves no human subjects, external data collection, released models, or deployment, and conforms to the NeurIPS Code of Ethics.

\item {\bf Broader impacts}
    \item[] Question: Does the paper discuss both potential positive societal impacts and negative societal impacts of the work performed?
    \item[] Answer: \answerNA{} % Replace by \answerYes{}, \answerNo{}, or \answerNA{}.
    \item[] Justification: \answerNA{}
    
\item {\bf Safeguards}
    \item[] Question: Does the paper describe safeguards that have been put in place for responsible release of data or models that have a high risk for misuse (e.g., pre-trained language models, image generators, or scraped datasets)?
    \item[] Answer: \answerNA{} % Replace by \answerYes{}, \answerNo{}, or \answerNA{}.
    \item[] Justification: \answerNA{}

\item {\bf Licenses for existing assets}
    \item[] Question: Are the creators or original owners of assets (e.g., code, data, models), used in the paper, properly credited and are the license and terms of use explicitly mentioned and properly respected?
    \item[] Answer: \answerNA{} % Replace by \answerYes{}, \answerNo{}, or \answerNA{}.
    \item[] Justification: \answerNA{}

\item {\bf New assets}
    \item[] Question: Are new assets introduced in the paper well documented and is the documentation provided alongside the assets?
    \item[] Answer: \answerNA{} % Replace by \answerYes{}, \answerNo{}, or \answerNA{}.
    \item[] Justification: \answerNA{}

\item {\bf Crowdsourcing and research with human subjects}
    \item[] Question: For crowdsourcing experiments and research with human subjects, does the paper include the full text of instructions given to participants and screenshots, if applicable, as well as details about compensation (if any)? 
    \item[] Answer: \answerNA{} % Replace by \answerYes{}, \answerNo{}, or \answerNA{}.
    \item[] Justification: \answerNA{}

\item {\bf Institutional review board (IRB) approvals or equivalent for research with human subjects}
    \item[] Question: Does the paper describe potential risks incurred by study participants, whether such risks were disclosed to the subjects, and whether Institutional Review Board (IRB) approvals (or an equivalent approval/review based on the requirements of your country or institution) were obtained?
    \item[] Answer: \answerNA{} % Replace by \answerYes{}, \answerNo{}, or \answerNA{}.
    \item[] Justification: \answerNA{}

\item {\bf Declaration of LLM usage}
    \item[] Question: Does the paper describe the usage of LLMs if it is an important, original, or non-standard component of the core methods in this research? Note that if the LLM is used only for writing, editing, or formatting purposes and does \emph{not} impact the core methodology, scientific rigor, or originality of the research, declaration is not required.
    %this research? 
    \item[] Answer: \answerNA{} % Replace by \answerYes{}, \answerNo{}, or \answerNA{}.
    \item[] Justification: \answerNA{}

\end{enumerate}

\end{document}